\documentclass{article}

\PassOptionsToPackage{sort, numbers, compress}{natbib}
\usepackage[final]{neurips_2025}

\usepackage[utf8]{inputenc} % allow utf-8 input
\usepackage[T1]{fontenc}    % use 8-bit T1 fonts
\usepackage{hyperref}       % hyperlinks
\usepackage{url}            % simple URL typesetting
\usepackage{booktabs}       % professional-quality tables
\usepackage{amsfonts}       % blackboard math symbols
\usepackage{nicefrac}       % compact symbols for 1/2, etc.
\usepackage{microtype}      % microtypography
\usepackage[dvipsnames]{xcolor}
\usepackage{color, colortbl}
\definecolor{Gray}{gray}{.95}
\definecolor{light_blue}{rgb}{0.91, 0.99, 0.99}
\newcolumntype{g}{>{\columncolor{Gray}}c}
\newcolumntype{q}{>{\columncolor{light_blue}}c}

\usepackage{adjustbox}
\usepackage{algorithm}
\usepackage{algorithmic}
\usepackage{booktabs}
\usepackage{cancel}
\usepackage{commath}
\usepackage{enumitem}
\usepackage{makecell}
\usepackage{microtype}
\usepackage{multirow} 
\usepackage{natbib}
\usepackage{nicefrac}
\usepackage{rotating}
\usepackage{soul} % for underlines
\usepackage{wrapfig}
\usepackage{xpatch}
\xpretocmd{\eqref}{Eq.~}{}{}

\usepackage{amsmath,amsfonts,bm}

\newcommand{\push}[2]{\ensuremath{{#2}_\sharp{#1}}}
\newcommand{\pushlong}[2]{\ensuremath{{#2}_*({#1})}}

\def\dominates{\ensuremath{\succcurlyeq}}
\def\quant{\ensuremath{\sQ_{\alpha}}}

\def\eqref#1{eq.~\ref{#1}}
\def\1{\bm{1}}

\def\vc{{\bm{c}}}

\def\vf{{\bm{f}}}

\def\vs{{\bm{s}}}

\def\vw{{\bm{w}}}

\DeclareMathAlphabet{\mathsfit}{\encodingdefault}{\sfdefault}{m}{sl}
\SetMathAlphabet{\mathsfit}{bold}{\encodingdefault}{\sfdefault}{bx}{n}

\def\gC{{\mathcal{C}}}
\def\gD{{\mathcal{D}}}

\def\gF{{\mathcal{F}}}

\def\gO{{\mathcal{O}}}
\def\gP{{\mathcal{P}}}

\def\gS{{\mathcal{S}}}

\def\gX{{\mathcal{X}}}
\def\gY{{\mathcal{Y}}}
\def\gZ{{\mathcal{Z}}}

\def\sQ{{\mathbb{Q}}}

\usepackage{amsthm}
\usepackage{amssymb}
\newtheorem{theorem}{Theorem}[section]

\newtheorem{lemma}[theorem]{Lemma}
\theoremstyle{definition}
\newtheorem{definition}[theorem]{Definition}

\newtheorem{proposition}[theorem]{Proposition}

\theoremstyle{remark}
\newtheorem{remark}[theorem]{Remark}
\newtheoremstyle{named}{}{}{\itshape}{}{\bfseries}{.}{.5em}{\thmnote{#3}#1}
\theoremstyle{named}

\newtheorem*{rep@theorem}{\rep@title}
\newcommand{\newreptheorem}[2]{%
	\newenvironment{rep#1}[1]{%
		\def\rep@title{#2 \ref{##1}}%
		\begin{rep@theorem}}%
		{\end{rep@theorem}}}

\newreptheorem{theorem}{Theorem}
\newreptheorem{lemma}{Lemma}
\newreptheorem{proposition}{Proposition}
\newreptheorem{corollary}{Corollary}

\def\P{\mathbb{P}}

\newcommand*{\prob}[1]{\mathbb{P}}

\def\N{\mathcal{N}}

\newcommand*{\ind}[1]{\mathbf{1}\left(#1\right)}
 
\newcommand{\E}{\mathbb{E}}

\newcommand{\R}{\mathbb{R}}

\DeclareMathOperator*{\argmin}{arg\,min}

\usepackage{centernot}

\title{Non-exchangeable Conformal Prediction with \\Optimal Transport: Tackling Distribution Shifts \\ with Unlabeled Data}

\author{%
    Alvaro H.C. Correia$^\dagger$ \quad
    Christos Louizos$^\dagger$ \\
  Qualcomm AI Research\thanks{Qualcomm AI Research is an initiative of Qualcomm Technologies, Inc. $^\dagger$Equal contribution.} \\
  Amsterdam, The Netherlands \\
  \texttt{\{acorreia, clouizos\}@qti.qualcomm.com} \\
}

\begin{document}

\maketitle

\begin{abstract}
    Conformal prediction is a distribution-free uncertainty quantification method that has gained popularity in the machine learning community due to its finite-sample guarantees and ease of use. Its most common variant, dubbed split conformal prediction, is also computationally efficient as it boils down to collecting statistics of the model predictions on some calibration data not yet seen by the model. Nonetheless, these guarantees only hold if the calibration and test data are exchangeable, a condition that is difficult to verify and often violated in practice due to so-called distribution shifts. The literature is rife with methods to mitigate the loss in coverage in this non-exchangeable setting, but these methods require some prior information on the type of distribution shift to be expected at test time. In this work, we study this problem via a new perspective, through the lens of optimal transport, and show that it is possible to estimate the loss in coverage and mitigate arbitrary distribution shifts, offering a principled and broadly applicable solution.
\end{abstract}

\section{Introduction}
Conformal prediction \citep{vovk2005algorithmic} (CP) works under the assumption that calibration and test data are exchangeable. Exchangeability is a weaker requirement than the more common i.i.d.~assumption but still implies that samples are identically distributed, which is hard to verify and ensure in practical applications. Therefore, it is important to develop conformal methods capable of adapting to potential distribution shifts or, at least, quantifying the gap in coverage caused by such shifts. In this paper, we study the effect of distribution shifts in conformal prediction through the lens of optimal transport, which proved instrumental in not only quantifying coverage gaps, but also alleviating them via a reweighting of the calibration data.

We start by introducing the notion of \emph{total coverage gap}: the expected coverage gap over all possible target coverage rates in $[0, 1]$. This metric captures the aggregate effect of distribution shift on conformal prediction and motivates the subsequent contributions, which build on this concept to provide theoretical bounds and practical strategies for mitigating the gap.
\begin{enumerate}[noitemsep,topsep=0pt,partopsep=0pt,parsep=5pt]
    \item We derive two new upper bounds to the total coverage gap, formulated in terms of optimal transport distances between the distributions of calibration and test nonconformity scores.
    \item We show that one of our upper bounds, which requires only unlabeled samples from the test distribution, can be used to learn weights $\vw = \{w_i\}_{i=1}^n$ for the calibration data. These weights can then be used in CP to reduce the gap in coverage during test time.
    \item We evaluate our methods on a (toy) regression task and on the ImageNet-C and iWildCam datasets. For the classification tasks we also consider the more challenging setting which includes covariate and label shift. 
\end{enumerate}
The paper is structured as follows. Section~\ref{sec:background} provides the necessary background for our main theoretical results. Section~\ref{sec:theoretical_results} presents our two new upper bounds on the total coverage gap, and Section~\ref{sec:learning} demonstrates their application to reduce the coverage gap. Finally, we review related work in Section~\ref{sec:related_work}, report experimental results in Section~\ref{sec:experiments}, and conclude in Section~\ref{sec:conclusion}.

\section{Background} \label{sec:background}
In this section, we lay out the background necessary for our main results. We start with a brief introduction to conformal prediction followed by an overview of optimal transport. 

\subsection{Conformal Prediction}
Conformal prediction~\cite{vovk2005algorithmic} is a framework to extract prediction sets from predictive models that satisfy finite-sample coverage guarantees under specific assumptions\footnote{Conformal prediction is often described as an uncertainty quantification method, but it may be more accurately viewed as an uncertainty representation technique: it conveys uncertainty through the size of prediction sets rather than assigning a numerical value to uncertainty.}. More precisely, consider the calibration set $(X_1,Y_1),\dots,(X_n,Y_n)$ drawn from an unknown distribution $P$ on $\mathcal{X}\times\mathcal{Y}$, formally $(X_1,Y_1,\dots,X_n,Y_n) \sim P^n$. Given this set, conformal prediction constructs prediction sets $\mathcal{C}(X_t)$ for new points $X_t \sim P$ such that the marginal coverage property holds for any $\alpha \in [0,1]$:
\begin{equation} \label{eq:main_cp}
	\P(Y_t \in \gC(X_t)) \geq 1- \alpha,
\end{equation}
where the probability is taken over the randomness of $\{(X_i, Y_i)\}_{i=1}^n$ and $(X_t, Y_t)$. The prediction set $\gC(X_t)$ is constructed using \emph{(non)conformity} scores, which quantify how well a sample fits within other samples in a set. One of the most common conformal prediction methods is that of split-conformal prediction (SCP)~\cite{papadopoulos2002inductive}. In SCP, a score $s(X_i, Y_i)$ is obtained for each point in a calibration set $\{(X_i, Y_i)\}_{i=1}^n$, and at test time $\gC(X_t)$ is constructed as
\begin{align*}
    \gC(X_t) = \left\{y \in \gY : s(X_t, y) \leq \quant\left(\{s(X_i, Y_i)\}_{i=1}^n \right) \right\},
\end{align*}
with $\quant$ the $1-\alpha$ quantile of the empirical distribution defined by the set of scores $\{s(X_i, Y_i)\}_{i={1}}^n$.

The main assumption for the marginal coverage guarantee to hold is that of \emph{exchangeability}; the new point $(X_t, Y_t)$ needs to be exchangeable with the points in the calibration set $\{(X_i, Y_i)\}_{i=1}^n$, i.e., it should follow the same distribution $P(X, Y)$. Unfortunately, violations of the exchangeability assumption are all too common \citep{koh2021wilds} and the naive application of standard SCP when $(X_t, Y_t)$ comes from another distribution $Q(X, Y)$ could produce misleading prediction sets that do not achieve the desired coverage rate \citep{vovk2005algorithmic,tibshirani_conformal_2019,barber2023conformal}. Proper usage of conformal prediction in these settings requires methods to (i) quantify the coverage gap caused by the distribution shift, and (ii) mitigate the effect of the shift on the CP procedure itself to get as close as possible to the target coverage rate. 

We refer the reader to \citep{shafer_tutorial_2008,angelopoulos_gentle_2021} for more thorough introductions to conformal prediction.

\subsection{Optimal Transport: Couplings and Wasserstein Distance}
Consider a complete and separable metric space $(\gZ, c)$, where $c: \gZ \times \gZ \rightarrow \R$ is a metric. Let $\gP_p(\gZ)$ be the set of all probability measures $P$ on $(\gZ, c)$ with finite moments of order $p \geq 1,$ i.e., $\int_{\gZ} c(z_0, z)^pdP(z) < \infty$ for some $z_0 \in \gZ.$ 
The $p$-Wasserstein distance is a metric on $\gP_p(\gZ)$ that is defined for any measures $P$ and $Q$ in $\gP_p(\gZ)$ as
\begin{equation} \label{eq:p-was}
	W_p(P,Q) = \left( \inf_{\pi \in \Gamma(P,Q)} \int_{\gZ \times \gZ} c(z, z')^p d\pi(z, z') \right)^{1/p}
\end{equation}
where $\Gamma(P,Q)$ denotes the collection of all measures on $\gZ \times \gZ$ with marginals $P$ and $Q$. We refer to any probability measure in 
$\Gamma(P,Q)$ as a coupling of $P$ and $Q$ and use $\pi^*(P, Q)$ to denote $p$-Wasserstein optimal couplings, i.e., any coupling that attains the infimum in (\ref{eq:p-was}).

Importantly, Wasserstein distances are also defined for discrete measures, and empirical measures in particular. Let $\hat P_n$ and $\hat Q_m$ denote the empirical distributions of samples $\{z_i\}_{i=1}^n, z_i \sim P$ and $\{z'_j\}_{j=1}^m, z'_j \sim Q,$ which induce empirical measures $\hat P_n = \sum_{i=1}^n \delta_{z_i}$ and $\hat Q_m = \sum_{j=1}^m \delta_{z'_j}.$ In that case, a coupling can be captured by a matrix $\Gamma$, with $\Gamma_{i,j}$ the mass to be transported from $z_i$ to $z'_j$. Similarly, for empirical measures the cost function $c$ induces a cost matrix with $C_{i,j} = \norm{z_i - z'_j}^p$ such that the transportation problem is given by 
\begin{equation*}
    \nonumber \min_\Gamma \sum_{i,j} C_{i,j} \Gamma_{i,j} \qquad \text{subject to} \quad \sum_{i=1}^n \Gamma_{i,j} = 1/m \,\,\, \forall j \in [\![m]\!], \,\,\,  \sum_{j=1}^m \Gamma_{i,j} = 1/n \,\,\, \forall i \in [\![n]\!].
\end{equation*}

In this work, we will be concerned with the distribution over nonconformity scores, which are typically one-dimensional. In this case, the $p$-Wasserstein simplifies to 
\begin{equation}
	W_p(P,Q) = \left( \int_0^1 \norm{F_P^{-1}(q) - F_Q^{-1}(q)}^p dq \right)^{1/p},
\end{equation}
where $F_P^{-1}$ (resp. $F_Q^{-1}$) is the quantile function, i.e., the inverse of the cumulative distribution function (CDF) $F_P$ (resp. $F_Q$) under measure $P$ (resp. $Q$). For $p=1$, we can also express the Wasserstein distance in terms of the respective CDFs
\begin{equation} \label{eq:w1_cdf}
   W_1(P, Q) = \int_\R \left|F_P(z) - F_Q(z)\right| dz.
\end{equation}
In this paper, we focus on the 1-Wasserstein distance. Our main results rely solely on the definitions and properties outlined here, plus basic properties like the triangle inequality. Nevertheless, the interested reader will be well served by the excellent introductions to optimal transport in \cite{villani2008optimal,peyre2019computational}. 
\section{Theoretical Results} \label{sec:theoretical_results}
In this section, we define the notion of total coverage gap and introduce our main theoretical results that allow us to upper bound it. For the sake of conciseness we defer the proofs to Appendix~\ref{app:theo}.

We start by laying out the necessary notation. Let $X \in \gX$ and $Y \in \gY$ be input and output variables and $s: \gX \times \gY \rightarrow \R$ be a (non)conformity score function. Regardless of the type of distribution shift (e.g., covariate or label shifts) its effect on the conformal prediction guarantees will manifest itself in the distribution over calibration and test scores. Therefore, in this paper we will directly manipulate the distribution of scores  $S = s(X, Y)$, and to that end we use $\push{P}{s}$ and $\push{Q}{s}$ to denote the calibration and test distributions over the scores, i.e., $\push{P}{s} = \pushlong{P}{s}$ is the pushforward measure of $P$ by $s$. When $\push{P}{s}$ is absolutely continuous with respect to the Lebesgue measure on $\mathbb{R}$, we write its density as $p_{\push{P}{s}}$. We also use subscripts to distinguish between scores observed during calibration $S_c$ and at test time $S_t$. As usual, we will use uppercase letters for random variables and lowercase letters for their realizations, e.g., $S_t=s_t$. 
We reserve calligraphic letters for sets, e.g., $\gS_c = \{s(X_i, Y_i)\}_{i=1}^n$ is the set of calibration scores obtained from a sample of size $n$ drawn from $P^n$, and $\mathcal{C}(X_t)$ denotes a prediction set for test variable $X_t$.

\subsection{Total Variation Distance Bound}
With this notation, we write the coverage under $P$ and $Q$ as
\begin{align*} 
    \nonumber P(Y_t \in&~\gC(X_t)) = P(S_t \leq \quant(\gS_c)) = \E_{S_t \sim \push{P}{s}} \left[ \E_{\gS_c \sim \push{P^n}{s}} [\ind{S_t \leq \quant(\gS_c)}] \right] \\
    \nonumber Q(Y_t \in&~\gC(X_t)) = Q(S_t \leq \quant(\gS_c)) = \E_{S_t \sim \push{Q}{s}} \left[ \E_{\gS_c \sim \push{P^n}{s}} [\ind{S_t \leq \quant(\gS_c)}] \right],
\end{align*}
with $\quant(\gS_c)$ the $1{-}\alpha$ quantile of the empirical distribution defined by a set of calibration scores $\gS_c$.

In general, one cannot guarantee valid coverage under arbitrary test distributions $Q$,  i.e., we cannot ensure $Q(Y_{test} \in \gC(X_{test})) \geq 1-\alpha$. Therefore, it is important to quantify the gap in coverage induced by the change in distribution from $P$ to $Q$. To that end, let $\Delta(\alpha)$ denote the coverage gap for a specific $\alpha$ value 
\begin{multline*} 
    \nonumber \Delta_{P,Q}(\alpha) := \big| P(S_t \leq \quant(S_c)) - Q(S_t \leq \quant(S_c)) \big |  \\
    \nonumber = \biggr|  \E_{S_t \sim \push{P}{s}} \left[ \E_{\gS_c \sim \push{P^n}{s}} \!\left[\ind{S_t \leq \quant(\gS_c)}\right] \right] \!\! - \E_{S_t \sim \push{Q}{s}} \left[ \E_{\gS_c \sim \push{P^n}{s}} \!\left[\ind{S_t \leq \quant(\gS_c)}\right] \right] \biggr|,
\end{multline*}
where $\ind{\cdot}$ is the indicator function.
It is easy to show the coverage gap is upper bounded by the total variation distance. We first restate the following well-known result for the total variation between two distributions $P$ and $Q$ (see e.g. \citet{farinhasnon}):
\begin{equation*}
	D_{TV}(P, Q) \geq |\E_P[g] - \E_Q[g]|,
\end{equation*}
for some function $g$ such that $|g| \leq 1$. It suffices to take $g$ as $g(x) = \E_{\gS_c \sim \push{P^n}{s}} [\ind{x \leq \quant(\gS_c)}]$, which is clearly bounded with $|g(x)| \leq 1$ for all $x$, to get $\Delta_{P,Q}(\alpha) \leq D_{TV}(P, Q).$
Unfortunately, estimating the total variation distance between $P$ and $Q$ without access to their respective densities is impractical. Instead, in the following we will propose two different ways to get around this difficulty and effectively quantify the coverage gap. In both cases, we leverage optimal transport, which defines valid distances even for empirical measures, i.e., when we only have access to $P$ and $Q$ via samples. 
\subsection{Upper Bound on the Total Coverage Gap}
We begin by defining the total coverage gap as follows.
\begin{definition}[Total coverage gap] The expected coverage gap across all possible values $\alpha \in [0, 1]$ given by
\begin{equation*}
    \Delta_{P,Q} := \int_0^1 \Delta_{P,Q}(\alpha)d\alpha = \E_{p(\alpha)}[\Delta(\alpha)],
\end{equation*}
with $p(\alpha)$ being the uniform distribution in $[0, 1]$.
\end{definition}
The following result establishes that the total coverage gap between $P$ and $Q$ is upper bounded by a weighted CDF distance and the 1-Wasserstein distance between them, $W_1(P, Q).$

\begin{theorem}\label{theo:tcgap}
Let $P$ and $Q$ be probability measures on $\gX \times \gY$ with $\push{P}{s}$ and $\push{Q}{s}$ their respective pushforward measures by a score function $s: \gX \times \gY \rightarrow \R$. Assume $\push{P}{s}$ is absolutely continuous with respect to the Lebesgue measure with density $p_{\push{P}{s}}(s_c)$. Then the total coverage gap can be upper bounded as follows
\begin{align}
    \Delta_{P,Q} &\leq \int_\R p_{\push{P}{s}}(s_c) \left| F_{\push{P}{s}}(s_c) {-} F_{\push{Q}{s}}(s_c) \right|  ds_c \label{eq:density-weighted} \\ 
    &\leq \Big( \sup_{s_c \in \R}p_{\push{P}{s}}(s_c) \Big) W_1(\push{P}{s}, \push{Q}{s}). \label{eq:w1-bound}
\end{align}
\end{theorem}

Naturally, the upper bound of Theorem~\ref{theo:tcgap} is tight if there is no distribution shift, in which case $W_1(\push{P}{s}, \push{Q}{s})$ evaluates to zero and the coverage gap is also zero by definition.
Both (\ref{eq:density-weighted}) and (\ref{eq:w1-bound}) are valid upper bounds to the total coverage gap that are easy to compute in practice. It suffices to estimate the density of calibration scores---e.g., via kernel density estimation (KDE)---and compute the 1-Wasserstein distance in (\ref{eq:w1-bound}) or the difference of CDFs in (\ref{eq:density-weighted}), all of which are easily computable from samples, especially since nonconformity scores are typically unidimensional. 
\subsection{Upper Bound on the Total Coverage Gap without Labels}
While the above bounds are informative, they come with one \textit{crucial} drawback; they require \textit{labeled} samples from $Q$, which might be hard to obtain in practice. To overcome the need for labels, we present another upper bound to the total coverage gap that can be computed with unlabeled data from the test distribution $Q$.
The main insight is that, although we may not have access to the score of the ground truth label, in the classification setting, we generally know the scores of \emph{all} possible labels. This gives us meaningful information on the distribution of scores under the shifted test distribution $Q$, which we use to construct auxiliary distributions $\push{Q^{\downarrow}}{s}$ and $\push{Q^{\uparrow}}{s}$, whose CDFs satisfy $F_{\push{Q^{\uparrow}}{s}}(t) \leq F_{\push{Q}{s}}(t) \leq F_{\push{Q^{\downarrow}}{s}}(t)$ for all $t \in \R$. 
This sandwiching of the CDF of $F_{\push{Q}{s}}(t)$ corresponds to a stochastic dominance relationship, denoted $\push{Q^{\uparrow}}{s} \dominates \push{Q}{s} \dominates \push{Q^{\downarrow}}{s}$, and allows us to construct bounds in the form of (\ref{eq:density-weighted}) and (\ref{eq:w1-bound}) even without access to the unknown $\push{Q}{s}$.

\begin{theorem} \label{theo:unlabeled_tcgap}
    Let $P$ and $Q$ be two probability measures on $\gX \times \gY$ with $\push{P}{s}$ and $\push{Q}{s}$ their respective pushforward measures by the score function $s: \gX \times \gY \rightarrow \R$. Assume $\push{P}{s}$ is absolutely continuous with respect to the Lebesgue measure with density $p_{\push{P}{s}}(s_c)$. Further let $\push{Q^{\downarrow}}{s}$ and $\push{Q^{\uparrow}}{s}$ be such that $\push{Q^{\uparrow}}{s} \dominates \push{Q}{s} \dominates \push{Q^{\downarrow}}{s}$. Then, we have that
    \begin{align}
        \nonumber \Delta_{P, Q} &\leq \frac{1}{2} \int p_{\push{P}{s}}(s_c) \bigg(\left|F_{\push{P}{s}}(s_c) - F_{\push{Q^\uparrow}{s}}(s_c) \right| + \left|F_{\push{P}{s}}(s_c) - F_{\push{Q^\downarrow}{s}}(s_c) \right| \\ & \hskip0.525\textwidth + F_{\push{Q^{\downarrow}}{s}}(s_c) - F_{\push{Q^{\uparrow}}{s}}(s_c)\bigg) ds_c \label{eq:unlabeled-density-weighted} \\
        & \leq \frac{1}{2}\Big( \sup_{s_c \in \R}p_{\push{P}{s}}(s_c) \Big) \bigg(W_1(\push{P}{s}, \push{ Q^{\uparrow}}{s}) + W_1(\push{P}{s}, \push{ Q^{\downarrow}}{s}) + \E_{\push{ Q^{\uparrow}}{s}}[S] - \E_{\push{Q^{\downarrow}}{s}}[S]\bigg) \label{eq:unlabeled-w1-bound}.
    \end{align}
\end{theorem}

Theorem~\ref{theo:unlabeled_tcgap} tells us that we can upper bound the total coverage gap between the calibration distribution $P$ and an unknown test distribution $Q$, if we can somehow find two auxiliary distributions over the test scores $\push{Q}{s}$, such that $\push{Q^{\uparrow}}{s} \dominates \push{Q}{s} \dominates \push{Q^{\downarrow}}{s}$. 
Fortunately, in the classification setting, we typically evaluate the scores of all possible classes (e.g., by computing the probability of all classes with a softmax activation). Thus, although the true score $s(x, y)$ is not observed, we know it must be contained in the set $\vs(x) = \{s(x, y') : y' \in \gY \}$. We can then use a set of $m$ unlabeled samples from $Q$ and their corresponding scores $\{\vs(x_i)\}_{i=1}^m$ to construct empirical distributions $\push{\hat Q_m^{\uparrow}}{s}$ and $\push{\hat Q_m^{\downarrow}}{s}$ with the required stochastic dominance relation to the unknown $\push{\hat Q_m}{s}.$
A natural solution is to take the minimum and maximum scores of each instance $x_i$ to get the following empirical distributions
\begin{equation} \label{eq:q_min_q_max}
    \push{\hat Q^{\min}_m}{s} := \frac{1}{m}\sum_{i=1}^{m} \delta_{\min \vs(x_i)} \qquad \push{\hat Q^{\max}_m}{s} := \frac{1}{m}\sum_{i=1}^{m} \delta_{\max \vs(x_i)},
\end{equation}
with $\ind{\cdot}$ the indicator function and $\delta_z$ the delta function at a value $z$. It is easy to see the empirical distributions obey $\push{\hat Q^{\max}_m}{s} \dominates \push{\hat Q}{s} \dominates \push{\hat Q^{\min}_m}{s}$ as needed. 

In the common setting where predictions come from a classifier $f$ that outputs class probabilities, a practical alternative we found effective is to construct the auxiliary distributions $\push{\hat Q^{U}_m}{s}$ and $\push{\hat Q^{f}_m}{s}$ as
\begin{equation} \label{eq:q_u_q_f}
    \push{\hat Q^{U}_m}{s} := \frac{1}{m}\sum_{i=1}^{m} \delta_{s(x_i, y'_i)}, y'_i \sim U(Y) \qquad \push{\hat Q^{f}_m}{s} := \frac{1}{m}\sum_{i=1}^{m} \delta_{s(x_i, y'_i)}, y'_i \sim Q_f(Y|x_i),
\end{equation}
where we take the score of a random label $j$ sampled either from a uniform distribution $U(Y)$ or from $Q_f(Y|X),$ the conditional distribution given by model $f.$ The motivation here is that $\push{\hat Q^{U}_m}{s}$ captures the scenario where the model is uninformative of the correct label, while $\push{\hat Q^{f}_m}{s}$ reflects the scenario in which the model perfectly captures the true distribution $Q(Y|X)$. Although we cannot guarantee $\push{\hat Q^{U}_m}{s} \dominates \push{\hat Q}{s} \dominates \push{\hat Q^{f}_m}{s}$, we empirically observe this relation to hold in most cases (see Fig.~\ref{fig:emp_cdfs}).

Naturally, the tightness of the upper bounds in Theorem~\ref{theo:unlabeled_tcgap} depend heavily on how close $\push{Q^{\downarrow}}{s}$ and $\push{Q^{\uparrow}}{s}$ are to the unknown $\push{Q}{s}$. In the absence of prior knowledge about the nature of the distribution shift, the best we can do is rely on the general auxiliary distributions described above, which may yield relatively loose bounds. Nevertheless, the bounds constructed using $(\push{\hat Q^{\min}_m}{s}, \push{\hat Q^{\max}_m}{s})$ or $(\push{\hat Q^f}{s}, \push{\hat Q^U}{s})$ serve as effective optimization objectives for reducing the coverage gap in practice, through a reweighting of the calibration samples, as we explain in the following section.

Before proceeding, we note that upper bounds on the coverage gap can also be derived for restricted ranges of the miscoverage rate $\alpha$. In Appendix~\ref{app:restricted_tcgap}, we present a bound for $\alpha$ ranging between $\alpha^-$ and $\alpha^+$, with $0 \leq \alpha^- \leq \alpha^+ \leq 1$. Additionally, Appendix~\ref{app:specific_alpha} provides a bound for a fixed miscoverage rate $\alpha$, denoted $\Delta_{P,Q}(\alpha)$. While these bounds are less effective as optimization objectives, they offer useful theoretical insights and are detailed in Appendix~\ref{app:theo}.

\begin{figure}
\centering
\begin{minipage}[t]{.46\textwidth}
 \centering
    \includegraphics[width=\linewidth]{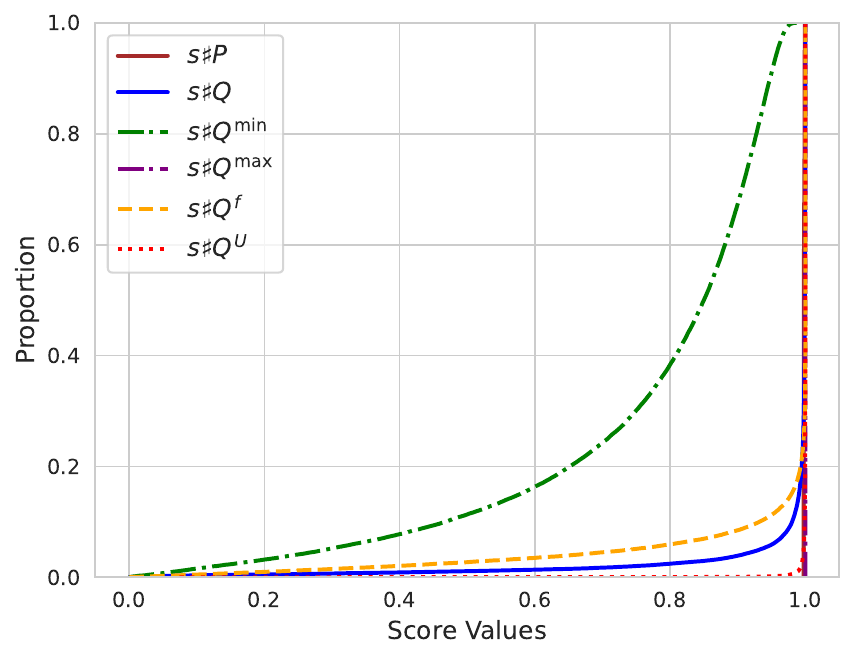}
    \vspace{-0.5cm}
    \caption{Empirical CDFs of nonconformity scores in ImageNet-C Gaussian noise under the calibration $\push{\hat P}{s}$, test $\push{\hat Q}{s}$, and auxiliary distributions. We can visually verify $\push{\hat Q^{\max}}{s} \dominates \push{\hat Q}{s} \dominates \push{\hat Q^{\min}}{s}$ and $\push{\hat Q^{U}}{s} \dominates \push{\hat Q}{s} \dominates \push{\hat  Q^{f}}{s}$.
    }
    \label{fig:emp_cdfs}
\end{minipage}\qquad
\begin{minipage}[t]{.46\textwidth}
  \centering
    \includegraphics[width=\linewidth]{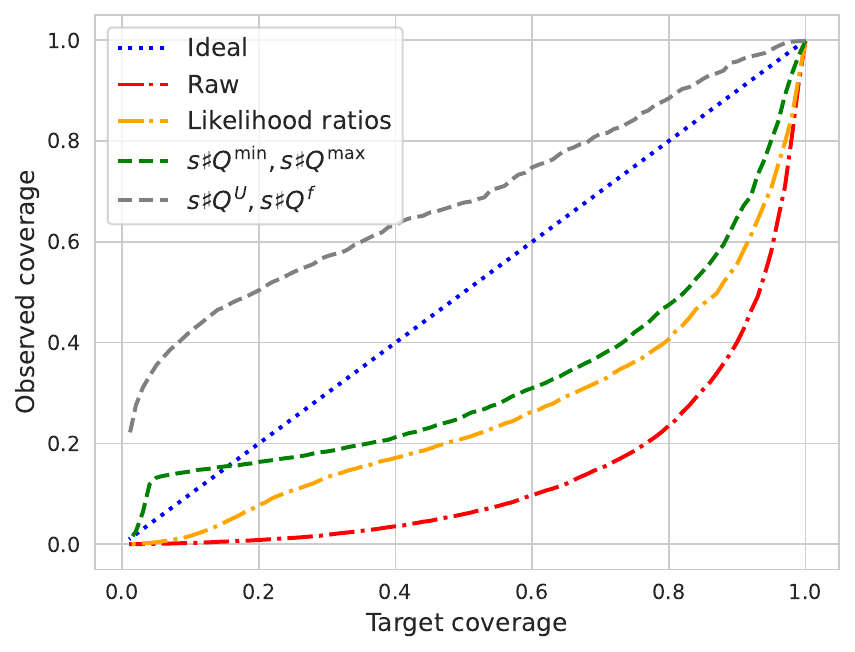}
    \vspace{-0.5cm}
    \caption{
        Total coverage gap in ImageNet-C Fog with weights learned via likelihood ratio estimation (orange), optimal transport with $(\push{\hat Q^{\min}}{s}$, $\push{\hat Q^{\max}}{s})$ in green, and $(\push{\hat Q^{f}}{s}$, $\push{\hat Q^{U}}{s})$ in gray.
    }
    \label{fig:emp_coverage}
\end{minipage}
\end{figure}

\section{Learning} \label{sec:learning}
In the conformal prediction literature, it is common to address non-exchangeability by reweighing the calibration points \citep{tibshirani_conformal_2019,barber2023conformal}. In practice, this implies that the quantile of the scores is computed on a weighted empirical distribution of the calibration scores $\push{P^\vw_n}{s}$ with the following empirical CDF
\begin{equation} \label{eq:weighted_cal_data}
    \push{\hat P^\vw_n}{s} = \sum_{i=1}^n w_i \delta_{s(x_i, y_i)},
\end{equation}
$w_i \geq 0$ are properly normalized weights associated with calibration samples $(x_i, y_i)$. In the case of covariate shifts, \citet{tibshirani_conformal_2019} show that we can recover proper coverage by setting the weights in (\ref{eq:weighted_cal_data}) proportionally to the likelihood ratio, i.e., $w_i \propto \nicefrac{dQ(x_i)}{dP(x_i)}$. \citet{barber2023conformal} also rely on a weighted empirical distribution in the form of (\ref{eq:weighted_cal_data}) but assume the weights to be fixed based on some prior knowledge of the likely deviations from exchangeability.

Motivated by these ideas, we propose instead to learn the distribution $\push{P^\vw_n}{s}$ directly by minimizing the upper bounds of Theorem~\ref{theo:unlabeled_tcgap} with respect to its weights $\vw$. We replace the Wasserstein and CDF distances in these bounds with their empirical counterparts (see Appendix~\ref{app:from_samples} for details), enabling optimization from samples. Specifically, we assume access to $n$ labeled samples $\{(x_i, y_i)\}_{i=1}^n$ from $P$ and $m$ unlabeled samples $\{x_i\}_{i=n+1}^{n+m}$ from the test distribution $Q$. These are used to construct $\push{\hat Q^\downarrow}{s}$ and $\push{\hat Q^\uparrow}{s}$. As discussed in Section~\ref{sec:theoretical_results}, two alternatives we consider are to take the pair $(\push{\hat Q^{\min}}{s}, \push{\hat Q^{\max}}{s})$ as in (\ref{eq:q_min_q_max}) or $(\push{\hat Q^f}{s}, \push{\hat Q^U}{s})$ as in (\ref{eq:q_u_q_f}). However, other constructions, potentially leveraging prior information about the distribution shift or the application domain, are possible, as long as $\push{\hat Q^{\uparrow}}{s} \dominates \push{\hat Q}{s} \dominates \push{\hat Q^{\downarrow}}{s}$. We note that, although using unlabeled samples from the test distribution is uncommon in CP, in many cases it is easy to collect such samples in practice.

Evaluating the bound of Theorem~\ref{theo:unlabeled_tcgap} 
admits an efficient exact solution for empirical distributions: it suffices to sort the samples ($n$ from $P$ and $m$ from $Q$) to compute the difference between the empirical CDFs. 
Crucially, when computing weighted empirical CDFs as in (\ref{eq:weighted_cal_data}), the weights $\vw$ only come into play after the score values are sorted, and thus the operation is trivially differentiable with respect to $\vw$, with no relaxation needed. 
Finally, we estimate the density of $p_{\push{P}{s}}$ by fitting a Gaussian kernel density estimator (KDE) to the calibration scores. It is easy to fit Gaussian KDEs to weighted samples, and the estimated density is also differentiable with respect to the weights.

Having established how to evaluate the bounds efficiently and differentiate through the weighting, we now turn to how these weights are parameterized. We consider two strategies:
\begin{itemize}
    \item \textbf{Free-form weights}. We directly optimize a set of unnormalized weights $\{\tilde w_i\}_{i=1}^n$, each one tied to a specific calibration point $(x_i, y_i)$. After optimization, these weights are normalized to form the weighted empirical distribution in (\ref{eq:weighted_cal_data}). This method is simple and effective, offering maximum flexibility for a fixed calibration set, but it remains restricted to that set.
    \item \textbf{Learnable weight function}. Alternatively, we learn a parametric function
    $w_\theta:\R\to\R_{\ge 0}$ with $\tilde w_i = w_\theta\big(s(x_i,y_i)\big)$
    where $\theta$ are the function parameters (e.g., a small neural network). Unlike the free-form approach, this formulation allows computing weights for additional calibration points or test candidates, thereby recovering the standard weighted split-CP setting and aligning with the importance-weighting principle of \citet{tibshirani_conformal_2019}.
\end{itemize}

Both parametrizations involve trade-offs. We focus most of our analysis on free-form weights because they align naturally with our bounds-based objectives, require only simple differentiable operations, and avoid extra modeling assumptions. This makes them stable and data-efficient in the small-sample regime, which is our primary concern. Nonetheless, these weights are tied to the specific calibration samples used during training, so the same set must be retained for calibration. This introduces dependencies among calibration points and breaks exchangeability with the test set. Under distribution shift, exchangeability is already compromised, so this violation is less critical. In this setting, the focus naturally shifts from preserving exchangeability to mitigating its effects, which is exactly what we achieve by optimizing calibration weights directly. Weight functions, by contrast, offer a more general solution: by mapping scores to weights through a parametric model, they can assign weights to unseen calibration points and recover the standard weighted split-CP setting. This flexibility comes at a cost: training the model requires additional labeled samples from $P$ and careful specification of its architecture. Despite these differences, both approaches achieve comparable empirical performance (see experiments in Section~\ref{sec:experiments} and Appendix~\ref{app:experiments}).

See Algorithm~\ref{alg:ours} for a sketch of how we optimize the total coverage gap from Theorem~\ref{theo:unlabeled_tcgap}. A more complete algorithm, including how this optimization fits into split CP is given in Algorithm~\ref{alg:ours-end2end}, in Appendix~\ref{app:experiments}. In practice, during optimization unnormalized weights $\{\tilde w_i\}$ are mapped to normalized weights via a softmax, regardless of whether $\tilde w_i$ come from a learnable vector or a weight function $w_\theta$. This normalization ensures differentiability and proper scaling before computing the empirical distribution used to evaluate the 1-Wasserstein or weighted-CDF bounds and fit the KDE. The computational cost remains the same in both cases: evaluating the 1-Wasserstein distance requires $\gO((m+n)\log(m+n))$ for sorting, while fitting a Gaussian KDE on $n$ calibration samples costs $\gO(k \cdot n)$, where $k$ is the number of evaluation points.

\begin{algorithm}[t]
\caption{Learning Weights for Non-exchangeable Conformal Prediction via Optimal Transport}
\label{alg:ours}
\begin{algorithmic}
    \STATE \textbf{Input:}
    \STATE \hspace{1em} $n$ labeled samples $\{(x_i, y_i)\}_{i=1}^n$ from $P$
    \STATE \hspace{1em} $m$ unlabeled samples $\{x_j\}_{j=n+1}^{n+m}$ from $Q$
    \STATE \hspace{1em} score function $s$
    \STATE Initialize unnormalized weights $\tilde \vw = \{\tilde w_i\}_{i=1}^n$ or weight function $w_\theta$
    \STATE Compute calibration scores $\{s(x_i,y_i)\}_{i=1}^n$
    \STATE Compute test score vectors $\{\vs(x_j)\}_{j=n+1}^{n+m}$ \hfill {\small // includes all candidate labels $\vs(x)=\{s(x,y):y \in \gY\}$}
    \REPEAT
        \STATE Construct $\hat{Q}^{\downarrow}$ and $\hat{Q}^{\uparrow}$ from $\{s(x_j)\}$ \hfill {\small // e.g., $(\min,\max)$ or $(f,U)$}
        \STATE Compute normalized weights: $\vw = \mathrm{softmax}(\tilde \vw)$
        \STATE Fit KDE to $\{s(x_i,y_i)\}_{i=1}^n$ with weights $\vw$
        \STATE Update $\tilde \vw$ or $w_\theta$ to minimize (\ref{eq:unlabeled-density-weighted}) or (\ref{eq:unlabeled-w1-bound}) \hfill {\small // weighted-CDF or 1-Wasserstein bound}
    \UNTIL convergence
\end{algorithmic}
\end{algorithm}

\paragraph{Regression setting} In principle, our upper bound to the total coverage gap in Theorem~\ref{theo:unlabeled_tcgap} is directly applicable to regression tasks. The only caveat is that the true score $s(X_t, Y_t)$ might no longer be restricted to a finite set of known values, as in the classification setting, and designing $\push{Q^\downarrow}{s}$ and $\push{Q^\uparrow}{s}$ is more challenging. One can always construct these auxiliary distributions based on some prior information about the task or the underlying distribution shift, but more generally it is possible to leverage the regression-as-classification framework \cite{guhaconformal}, and again use $(\push{\hat Q^{\min}}{s}$, $\push{\hat Q^{\max}}{s})$ or $(\push{\hat Q^{f}}{s}$, $\push{\hat Q^{U}}{s}),$ as we do successfully in the experiment described in Section~\ref{sec:exp_reg}.
\section{Related Work} \label{sec:related_work}
Several works have studied conformal prediction in the non-exchangeable setting, especially in the context of time series, where the exchangeability assumption is violated by the very autoregressive nature of these problems \citep{xu2021conformal,xu2023conformal,jensen2022ensemble,zaffran2022adaptive,oliveira2024split, chernozhukov2018exact,gibbs2021adaptive,gibbs2024conformal}. Closer to our work, \citet{tibshirani_conformal_2019} and \citet{barber2023conformal} have also proposed to mitigate the coverage gap by reweighing the calibration samples. However, in \cite{tibshirani_conformal_2019} their weights only address covariate shifts and correspond to the unknown likelihood ratio $\nicefrac{dQ(x)}{dP(x)}$, which is hard to estimate in practice, especially under severe distribution shifts and the density chasm problem \citep{rhodes2020telescoping}. In our experiments, our methods compare favorably to learned likelihood ratios as proposed in \citep{tibshirani_conformal_2019}, attesting to the difficulty of learning accurate ratios.

To our knowledge, our methods and the reweighing scheme of \citep{barber2023conformal} are the only capable of tackling arbitrary distribution shifts in split conformal prediction. Still, their approach involves data-independent weights that must be designed a priori using some prior knowledge about the underlying distribution shift, whereas we learn appropriate weights directly from a few unlabeled samples from the test distribution. With the exception of the work of \citep{podkopaev2021distribution}, where label shift is also tackled via likelihood ratios in a similar fashion to \citep{tibshirani_conformal_2019}, most other works focus on covariate shift \citep{lei2021conformal,jin2023sensitivity,kasa2024adapting}. Among these, \citep{yang2024doubly,qiu2023prediction,gibbs2023conformal,kiyani2024length} are notable for tackling covariate shifts by approximating conditional coverage guarantees, i.e., by approximately satisfying $\P(Y_t \in \gC(X_t) | X_t) \geq 1- \alpha$. These achieve impressive results but are computationally expensive or limited to specific types of covariate shifts. 

Of special note is the work of \citep{gibbs2023conformal}, which proposes to adapt the conformal threshold for each test point, providing conditional coverage guarantees if the distribution shift comes from a known function class. In our experiments, where it is not clear how to define such function class, their method---implemented in its most general form via radial basis function (RBF) kernels---produced larger prediction sets than our methods at a much larger computational cost at test time. We also compare against entropy scaled conformal prediction or ECP \citep{kasa2024adapting}. ECP consists in dividing the threshold, i.e., $\quant(\gS_c)$, by the $(1{-}\alpha)$ quantile of the entropy over class predictions for the test points, with the intuition that high entropy (indicating high uncertainty) will decrease the threshold, leading to larger prediction sets. While this heuristic proved effective for covariate shift, the improvements in coverage were modest in the context of label shift. In contrast, our methods demonstrate greater robustness to different types of shift (see Table~\ref{tab:imagenetc_cov_big}).

Finally, our methods are closely related to the concurrent work of \citep{xu2025wassersteinregularized}, which also explores the relationship between the coverage gap and the $W_1(P,Q)$. However, their results are derived through a different approach and bound $\Delta_{P,Q}(\alpha)$ for any $\alpha$. Unfortunately, their bound does not depend on $\alpha$, and is thus loose for most target coverage rates. Moreover, in practice, their methods are only applicable to distribution shifts where the test distribution is a mixture of different calibration distributions. In contrast, our bound from Theorem~\ref{theo:unlabeled_tcgap} can be applied to any type distribution shift.  
\section{Experiments} \label{sec:experiments}
In this section, we describe and analyze a set of experiments designed to evaluate and validate our methods. In each of them, we have two sets of samples, $\gD_P$ distributed according to some calibration distribution $P$ and $\gD_Q$ according to some test distribution $Q$, with $P$ differing from $Q$ via some form of distribution shift. We divide $\gD_Q$ into two, with $\gD_Q^{(1)}$ reserved for fitting a density ratio estimator or learning the weights in our method as in Algorithm~\ref{alg:ours}, and $\gD_Q^{(2)}$ used for testing. When no pretrained model is available, we also split $\gD_P$, using $\gD_P^{(1)}$ for training a model and $\gD_P^{(2)}$ for calibration in split CP. Main experiments use 300 samples for $\gD_P^{(2)}$ and $\gD_Q^{(1)}$, with $\gD_P^{(2)}$ extended to 600 for weight functions (split evenly for fitting and calibration). Sample size effects are analyzed in Appendix~\ref{app:experiments}.

For regression tasks, we optimize the 1-Wasserstein variant of our bounds (\ref{eq:unlabeled-w1-bound}), while for image classification tasks, we adopt the weighted-CDF formulation (\ref{eq:unlabeled-density-weighted}). These choices reflect empirical findings; each variant performs best in its respective domain, as discussed in Appendix~\ref{app:design}. In all cases, we define nonconformity scores as one minus the probabilities assigned by the model.

\subsection{Regression setting with synthetic data} \label{sec:exp_reg}
We start with the synthetic data experiment proposed in \citep{yang2024doubly}, where the ground truth likelihood ratios are known. We use the regression-as-classification method of \citet{guhaconformal}, splitting the output space into 50 equally spaced bins. In Figure~\ref{fig:regression}, our methods significantly enhance coverage in most cases, with no notable difference between the two variants. In this low-dimensional setting, learning the likelihood ratios also proved effective, albeit with a slight tendency to under-cover. In contrast, our method exhibited a mild bias toward over-coverage. See Appendix~\ref{app:experiments} for experimental details.

\begin{figure}
    \centering
    \includegraphics[width=0.925\linewidth]{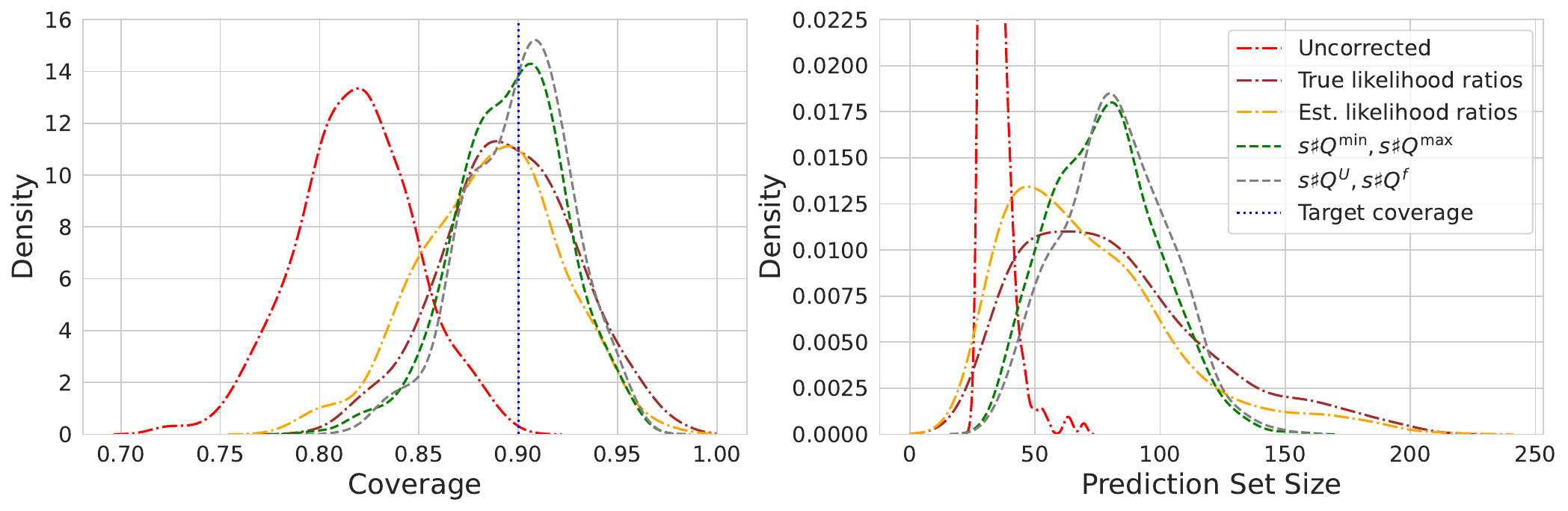}
    \caption{Distribution of coverage and prediction set sizes for the synthetic regression task across 500 simulations and target coverage rate of $90\%$ (blue vertical line). For ease of visualization, we plot the density estimated with a KDE fit to the 500 observations.}
    \label{fig:regression}
\end{figure} 

\subsection{Image classification} 
\paragraph{Imagenet-C}\label{sec:imagenetc}
We use the established ImageNet-C \citep{imagenetc} dataset to test our methods under the covariate and label shift settings. ImageNet-C contains a total of 15 covariate shifts at different severity levels, from $1$ (least severe) to $5$ (most severe), but no label shift. Since our methods apply to any type of distribution shift, we also simulate label shift in ImageNet-C. The details can be found in Appendix~\ref{app:experiments}. In Table~\ref{tab:imageclassification} (and its extended version in Appendix~\ref{app:experiments}) we observe a drastic drop in coverage for the uncorrected scores; from the target coverage of $90\%$, we drop to an average of $\sim80\%$ for severity level 1, and to an average of $\sim30\%$ for severity level 5. By introducing weighting to the calibration set, our OT methods improves coverage across the board, achieving similar coverage with and without label shift. In contrast, learned likelihood ratios provided modest improvements in coverage and the ECP method of \citep{kasa2024adapting} only produced competitive results in the absence of label shift. Finally, the method of \citep{gibbs2023conformal} produced good results in terms of coverage in ImageNet-C but at the cost of excessive large prediction sets (see Section~\ref{subsec:discussion} for a discussion on prediction set sizes). 

\paragraph{iWildCam}
We further use iWildCam~\citep{beery2020iwildcam} as one more dataset that contains natural distribution shifts. We can see in Table~\ref{tab:imageclassification} that the distribution shift incurs roughly a $10\%$ drop in coverage when not correcting the scores. While likelihood ratio weighting produce only modest improvements over the uncorrected scores, both of our OT settings improve coverage. The $(\min, \max)$ setting increased coverage by more than $7\%$ on average, while the $(f, U)$ setting improved coverage by almost $10\%$. Finally, the ECP method \citep{kasa2024adapting} got almost perfect coverage, even under label shift. Interestingly, the method of \citep{gibbs2023conformal} actually hurt coverage. This could be either because the distribution shift in iWildCam is not well captured by the class of shifts considered in their method (given by RBF kernels in this experiment) or due to severe class imbalance in iWildCam, which might hamper the optimization.

Notably, the results were consistent across parametrization choices, with both the free-form and weight-function variants yielding very similar performance overall. The only meaningful differences emerged under relatively mild distribution shifts: the weight-function approach performed better on iWildCam, while the free-form variant showed stronger results on ImageNet-C at severity level one.

\begin{table*}[!t]\centering
    \small
    \caption{
        \textbf{Average coverage and prediction set size on image classification tasks, with and without label shift}. 
        Results for uncorrected distributions, calibrating and testing on $Q$ (Oracle), likelihood ratios (LR), the methods of \citet{kasa2024adapting} and \citet{gibbs2023conformal}, and our methods with weighted-CDF objective (\ref{eq:unlabeled-density-weighted}), including $({\min}, {\max})$ and $(f, U)$ variants, and free-form (FF) and weight function (WF) parametrizations. The target coverage is set to 90\%. Extended version in Table~\ref{tab:imagenetc_cov_big}.
    }
    \label{tab:imageclassification}
    \resizebox{\linewidth}{!}{
        \begin{tabular}{ll cg cg cg cg}
            \toprule
            & &\multicolumn{2}{c}{iWildCam} &\multicolumn{2}{c}{ImageNet-C Sev.~1} &\multicolumn{2}{c}{ImageNet-C Sev.~3} &\multicolumn{2}{c}{ImageNet-C Sev.~5}  \\\cmidrule{3-10}
            & & Cov. & Size & Cov. & Size & Cov. & Size & Cov. & Size \\\midrule
            \multirow{9}{*}{\rotatebox[origin=c]{90}{no label shift}} & Uncorrected & $78.2_{\pm3.0}$ & $22.1_{\pm4.4}$ & $78.7_{\pm6.0}$ & $2.7_{\pm0.7}$ & $58.7_{\pm14.3}$ & $3.3_{\pm1.2}$ & $29.8_{\pm18.9}$ & $3.3_{\pm1.5}$   \\
            & Oracle  & $89.7_{\pm1.4}$ & $50.5_{\pm5.3}$ & $89.8_{\pm1.6}$ & $10.7_{\pm7.2}$ & $89.9_{\pm1.7}$ & $80.1_{\pm71.4}$ & $89.7_{\pm1.9}$ & $338.6_{\pm190.9}$   \\
            & LR  & $79.0_{\pm2.6}$ & $23.3_{\pm4.8}$ & $84.1_{\pm4.9}$ & $5.2_{\pm3.8}$ & $71.3_{\pm11.9}$ & $12.3_{\pm13.1}$ & $47.4_{\pm20.2}$ & $27.7_{\pm37.0}$   \\
            & \citeauthor{kasa2024adapting} & $95.2_{\pm1.2}$ & $79.7_{\pm9.0}$ & $96.6_{\pm1.1}$ & $42.4_{\pm19.2}$ & $93.7_{\pm3.2}$ & $119.3_{\pm65.4}$ & $86.1_{\pm8.4}$ & $282.0_{\pm159.1}$   \\
            & \citeauthor{gibbs2023conformal} & $67.7_{\pm6.0}$ & $109.6_{\pm6.3}$ & $88.9_{\pm2.8}$ & $548.5_{\pm36.7}$ & $85.6_{\pm3.6}$ & $635.7_{\pm66.3}$ & $84.9_{\pm5.6}$ & $754.4_{\pm102.7}$   \\
            \cmidrule{2-10}
            & FF $({\min}, {\max})$ & $85.2_{\pm4.1}$ & $37.5_{\pm10.9}$ & $91.4_{\pm3.9}$ & $15.1_{\pm9.7}$ & $88.2_{\pm7.7}$ & $63.9_{\pm41.3}$ & $71.3_{\pm17.0}$ & $128.5_{\pm93.0}$   \\
            & WF $({\min}, {\max})$ & $89.8_{\pm2.9}$ & $51.7_{\pm11.1}$ & $94.6_{\pm6.3}$ & $34.7_{\pm20.6}$ & $87.0_{\pm14.1}$ & $69.7_{\pm46.6}$ & $71.5_{\pm17.1}$ & $130.5_{\pm94.1}$   \\
            & FF $(f, U)$ & $88.2_{\pm3.4}$ & $46.2_{\pm11.9}$ & $93.0_{\pm3.0}$ & $19.2_{\pm11.4}$ & $90.3_{\pm5.9}$ & $79.8_{\pm48.9}$ & $79.5_{\pm12.1}$ & $205.7_{\pm164.4}$ \\
            & WF $(f, U)$ & $88.9_{\pm4.6}$ & $49.7_{\pm14.1}$ & $95.7_{\pm2.2}$ & $36.6_{\pm20.1}$ & $90.1_{\pm6.2}$ & $79.7_{\pm50.1}$ & $78.7_{\pm13.1}$ & $200.4_{\pm164.2}$ \\
            \midrule
            \multirow{9}{*}{\rotatebox[origin=c]{90}{with label shift}} & Uncorrected & $79.2_{\pm7.1}$ & $20.9_{\pm5.4}$ & $79.2_{\pm8.2}$ & $2.7_{\pm0.8}$ & $59.2_{\pm15.7}$ & $3.2_{\pm1.2}$ & $29.8_{\pm20.1}$ & $3.3_{\pm1.5}$   \\
            & Oracle  & $90.2_{\pm4.3}$ & $43.7_{\pm20.7}$ & $90.3_{\pm4.5}$ & $15.0_{\pm19.9}$ & $90.2_{\pm3.8}$ & $88.1_{\pm87.4}$ & $90.5_{\pm4.0}$ & $349.0_{\pm220.0}$   \\
            & LR  & $81.5_{\pm5.7}$ & $24.6_{\pm4.4}$ & $84.5_{\pm7.1}$ & $6.1_{\pm5.5}$ & $72.5_{\pm12.9}$ & $14.3_{\pm15.8}$ & $47.3_{\pm21.2}$ & $27.7_{\pm33.5}$   \\
            & \citeauthor{kasa2024adapting} & $88.6_{\pm6.3}$ & $35.9_{\pm17.1}$ & $79.2_{\pm8.1}$ & $2.7_{\pm0.8}$ & $59.7_{\pm15.2}$ & $3.4_{\pm1.4}$ & $31.2_{\pm19.9}$ & $4.1_{\pm2.3}$   \\
            & \citeauthor{gibbs2023conformal} & $38.0_{\pm15.6}$ & $44.9_{\pm23.5}$ & $88.6_{\pm3.8}$ & $552.3_{\pm46.0}$ & $85.5_{\pm4.6}$ & $638.6_{\pm70.2}$ & $84.8_{\pm5.8}$ & $753.5_{\pm100.4}$   \\
            \cmidrule{2-10}
            & FF $({\min}, {\max})$ & $83.1_{\pm5.8}$ & $22.1_{\pm12.6}$ & $91.0_{\pm5.7}$ & $14.3_{\pm10.2}$ & $88.4_{\pm8.6}$ & $62.6_{\pm41.4}$ & $72.0_{\pm18.8}$ & $127.9_{\pm94.4}$   \\
            & WF $({\min}, {\max})$ & $91.3_{\pm4.3}$ & $52.2_{\pm15.2}$ & $93.7_{\pm9.3}$ & $33.1_{\pm21.5}$ & $87.3_{\pm12.8}$ & $69.4_{\pm47.3}$ & $70.4_{\pm19.2}$ & $129.2_{\pm96.3}$   \\
            & FF $(f, U)$ & $88.2_{\pm4.7}$ & $33.5_{\pm12.9}$ & $93.5_{\pm3.6}$ & $20.7_{\pm14.6}$ & $90.3_{\pm6.8}$ & $78.9_{\pm51.7}$ & $79.8_{\pm13.9}$ & $206.0_{\pm168.3}$   \\
            & WF $(f, U)$ & $91.0_{\pm3.8}$ & $49.8_{\pm10.9}$ & $95.1_{\pm7.1}$ & $36.3_{\pm21.8}$ & $90.4_{\pm7.0}$ & $80.0_{\pm51.4}$ & $79.2_{\pm14.6}$ & $201.3_{\pm168.6}$ \\
            \bottomrule
        \end{tabular}
    }
\end{table*}

\subsection{Discussion} \label{subsec:discussion}
\paragraph{Prediction set sizes.} Our methods, like most other approaches to non-exchangeable CP, including the likelihood ratios proposed by \citet{tibshirani_conformal_2019}, and the weights introduced by \citet{barber2023conformal}, do not alter the observed test scores, thereby preserving the ranking of classes under the shifted distribution. Consequently, similar to these other works, our analysis focuses on coverage. The underlying CP algorithm (split CP) remains unchanged, and the variance in prediction set sizes is attributed to miscoverage. The best achievable performance mirrors what would be obtained if labeled samples from $Q$ were available; undercoverage results in smaller-than-optimal prediction sets, while overcoverage leads to larger-than-optimal sets. The only exception to this is the method of \citeauthor{gibbs2023conformal}, which aims for conditional coverage and thus changes the conformal threshold for each test point. This decouples coverage and prediction set size, but in our experiments %(see Table~\ref{tab:imagenetc} and Appendix~\ref{app:experiments}) 
their methods produced larger prediction sets than other methods despite getting close to the target coverage.

\paragraph{Bound variants.} In Theorems~\ref{theo:tcgap} and~\ref{theo:unlabeled_tcgap}, we have two flavors of upper bounds: one expressed in the terms of a weighted distance of CDFs, and another using the 1-Wasserstein distance. The former is always tighter and, likely for this reason, has shown superior performance in image classification tasks. Conversely, the 1-Wasserstein bounds, while generally looser, establish a more natural connection to optimal transport theory and has shown better empirical performance in regression tasks. This may be due to the weighting by $p_{\push{P}{s}}(s_c)$ in (\ref{eq:density-weighted}) and (\ref{eq:unlabeled-density-weighted}), which might complicate optimization.

\paragraph{Number of samples.} Our methods are effective across varying numbers of labeled samples from $P$ and unlabeled samples from $Q$. As shown in Appendix~\ref{app:influence_number_of_samples}, coverage improves with more samples as expected, but meaningful gains are observed even with as few as 30 samples from each distribution.

\paragraph{Limitations.} 
The tightness of the unlabeled upper bounds in Theorem~\ref{theo:unlabeled_tcgap} depends on the auxiliary distributions $\push{Q^\downarrow}{s}$ and $\push{Q^\uparrow}{s}$. When these closely approximate $\push{Q}{s}$, the resulting bounds are tight, and optimizing either (\ref{eq:unlabeled-density-weighted}) or (\ref{eq:unlabeled-w1-bound}) is predictably highly effective in reducing the coverage gap.
However, in the practically interesting setting of a general and unknown distribution shift we consider, the available choices for $\push{Q^\downarrow}{s}$ and $\push{Q^\uparrow}{s}$ are likely less informative and yield a necessarily looser bound. Having said that, the bounds computed with $(\push{\hat Q^{\min}}{s}$, $\push{\hat Q^{\max}}{s})$ and $(\push{\hat Q^{f}}{s}$, $\push{\hat Q^{U}}{s})$ still perform surprisingly well in reducing the coverage gap when used for reweighing the calibration data, demonstrating their practical value and broad applicability.
Nevertheless, care must be taken: in cases of minimal or no distribution shift, $\push{P}{s}$ may already offer better coverage than the solution to our objective. It is therefore advisable to first assess the presence of a distribution shift, potentially using unlabeled samples from $Q$, before applying our methods. We explore this issue further in Appendix~\ref{app:experiments}.
\section{Conclusion} \label{sec:conclusion}
In this work, we employ optimal transport theory to study the effect of distribution shifts on conformal prediction. Specifically, we derive upper bounds on the total coverage gap induced by a shift from the calibration distribution $P$ to the test distribution $Q$, expressed in terms of (weighted) CDF and Wasserstein distances. Recognizing that labeled examples from $Q$ are often unavailable in practice, we extend our analysis by leveraging the structure inherent in the nonconformity scores of the unlabeled test data. To this end, we construct auxiliary distributions $\push{\hat Q^\downarrow}{s}$ and $\push{\hat Q^\uparrow}{s}$, which enable label-free bounds of the coverage gap. Furthermore, we utilize these bounds as optimization objectives to learn importance weights over the calibration data. Empirically, our approach significantly reduces the coverage gap across a range of distribution shift settings.

\bibliography{bibliography}
\bibliographystyle{icml2025}

% \newpage
% \input{sections/neurips_checklist}

\appendix
\newpage
\section*{Broader Impact}
This work introduces new methods to enhance the coverage of conformal prediction under distribution shifts, which frequently occur in real-world applications. We believe our contributions will have a positive impact, encouraging practitioners to adopt uncertainty quantification techniques like conformal prediction, provided the underlying guarantees are well understood.

\section{Proofs and Additional Theoretical Results} \label{app:theo}
In this section, we provide detailed proofs of our new upper bounds to the coverage gap as well as extra theoretical results. We start by providing the proofs of Theorems~\ref{theo:tcgap} and \ref{theo:unlabeled_tcgap}.
\subsection{Bound to the Total Coverage Gap}
\begin{reptheorem}{theo:tcgap}
Let $P$ and $Q$ be probability measures on $\gX \times \gY$ with $\push{P}{s}$ and $\push{Q}{s}$ their respective pushforward measures by a score function $s: \gX \times \gY \rightarrow \R$. Assume $\push{P}{s}$ is absolutely continuous with respect to the Lebesgue measure with density $p_{\push{P}{s}}(s_c)$. Then the total coverage gap can be upper bounded as follows
\begin{align}
    \Delta_{P,Q} &\leq \int_\R p_{\push{P}{s}}(s_c) \left| F_{\push{P}{s}}(s_c) {-} F_{\push{Q}{s}}(s_c) \right|  ds_c \tag{\ref{eq:density-weighted}} \\ 
    &\leq \Big( \sup_{s_c \in \R}p_{\push{P}{s}}(s_c) \Big) W_1(\push{P}{s}, \push{Q}{s}). \tag{\ref{eq:w1-bound}}
\end{align}
\end{reptheorem}
\begin{proof}
\begin{flalign}
    \nonumber \Delta_{P,Q} &= \int_0^1 \biggr| \E_{\gS_c \sim \push{P^n}{s}} \left[ \E_{S_t \sim \push{P}{s}} [\ind{S_t \leq \quant(\gS_c)}] \right] \\ & \hspace{5cm} - \E_{\gS_c \sim \push{P^n}{s}} \left[ \E_{S_t \sim \push{Q}{s}} [\ind{S_t \leq \quant(\gS_c)}] \right] \biggr| d\alpha \\
    &= \int_0^1 \left|  \E_{\gS_c \sim \push{P^n}{s}} \left[ F_{\push{P}{s}}(\quant(\gS_c)) - F_{\push{Q}{s}}(\quant(\gS_c)) \right] \right| d\alpha \\
    &\leq \int_0^1\E_{\gS_c \sim \push{P^n}{s}} \left[  \left|  F_{\push{P}{s}}(\quant(\gS_c)) - F_{\push{Q}{s}}(\quant(\gS_c))  \right|  \right] d\alpha \label{eq:jensen} \\
    &= \E_{\gS_c \sim \push{P^n}{s}}\left[ \sum_{i=1}^n \frac{1}{n} \left| F_{\push{P}{s}}(s_{c}^{(i)}) {-} F_{\push{Q}{s}}(s_{c}^{(i)}) \right| \right] \label{eq:empirical} \\
    &= \E_{\gS_c \sim \push{P^n}{s}}\left[ \left| F_{\push{P}{s}}(S_c) {-} F_{\push{Q}{s}}(S_c) \right| \right] \label{eq:sample_mean} \\
    &= \int_\R p_{\push{P}{s}}(s_c) \left| F_{\push{P}{s}}(s_c) {-} F_{\push{Q}{s}}(s_c) \right|  ds_c  \\
    &\leq \left[\sup_{s_c\in \R}p_{\push{P}{s}}(s_c)\right] \int_\R \left| F_{\push{P}{s}}(s_c) {-} F_{\push{Q}{s}}(s_c) \right|  ds_c \label{eq:pull_max} \\
    &= \left[\sup_{s_c\in \R} p_{\push{P}{s}}(s_c) \right] W_1(\push{P}{s}, \push{Q}{s}). \label{eq:w1-bound_2}
\end{flalign}
where the first inequality in (\ref{eq:jensen}) is due to Jensen's inequality and (\ref{eq:empirical}) holds because the empirical quantile $\quant(\gS_c)$ must evaluate to one of the $n$ values in $\gS_c = \{s_{c}^{(1)}, \ldots, s_{c}^{(n)}\}$, each of which takes $\frac{1}{n}$ of the $[0,1]$ range. In (\ref{eq:empirical}), we have the expectation of the sample mean, which equals the expectation of the population as in (\ref{eq:sample_mean}). Lastly, (\ref{eq:pull_max}) holds because $p_{\push{P}{s}}(s_c)$ is a density and thus non-negative everywhere, and (\ref{eq:w1-bound_2}) follows directly from the definition of the 1-Wasserstein distance, as in (\ref{eq:w1_cdf}).
\end{proof}

\begin{remark}[On the expectation under a weighted measure]
The bound in Theorem~\ref{theo:tcgap} extends to any weighted calibration measure $P_\rho$ with density $\rho(s)p_{\push{P}{s}}(s)$, where $\rho : \R \to [0,\infty)$ satisfies
\begin{equation*}
    \int_{\R} \rho(s)p_{\push{P}{s}}(s) \, ds = 1.
\end{equation*}
In this case, the total coverage gap is upper bounded by
\begin{equation*}
\Delta_{P_\rho,Q} \le \int_{\R} \rho(s)p_{\push{P}{s}}(s)\,\big|F_{\push{P_\rho}{s}}(s) - F_{\push{Q}{s}}(s)\big| \, ds
\le \Big( \sup_{s \in \R} \rho(s)p_{\push{P}{s}}(s) \Big) W_1(\push{P_\rho}{s}, \push{Q}{s}).
\end{equation*}
The proof follows identically by replacing $p_{\push{P}{s}}$ with $\rho(s)p_{\push{P}{s}}(s)$ in the argument.
\end{remark}

\begin{remark}[On practical weighting]
In experiments (Section~4), weights are applied at the sample level, forming a weighted empirical measure rather than a continuous density. These weights are not globally normalized; instead, normalization is enforced through a softmax or similar constraint during optimization. Our theoretical result assumes a normalized weighting function $\rho$, ensuring $P_\rho$ is a probability measure. This is a stronger condition than what is used in practice, but the empirical approach approximates this normalization. Moreover, the bound holds for the empirical distribution defined by a calibration dataset as shown in Proposition~\ref{prop:weighted_bound} below.
\end{remark}

\begin{proposition}[Empirical weighted bound] \label{prop:weighted_bound}
Let $\push{\hat{P}^\vw_n}{s} = \sum_{i=1}^n w_i \,\delta_{s_c^{(i)}}$ be the weighted empirical measure over
calibration scores $\{s^{(i)}\}_{i=1}^n$, with weights $w_i \ge 0$ and $\sum_{i=1}^n w_i = 1$.
Let $F_{\push{\hat{P}^\vw_n}{s}}$ and $F_{\push{Q}{s}}$ denote the CDFs of $\push{\hat{P}^\vw_n}{s}$ and an arbitrary test distribution over scores $\push{Q}{s}$, respectively. Define the weighted quantile function
$\quant^\vw(\alpha) := F_{\push{\hat{P}^\vw_n}{s}}^{-1}(\alpha)$ and the empirical weighted coverage gap
\begin{equation*}
\widehat{\Delta}_{w,Q}
:= \int_0^1 \left| \E_{\gS_c \sim \push{\hat{P}^\vw_n}{s}} \left[ \E_{S_t \sim \push{P}{s}} \big[\ind{S_t \le \quant^\vw(\alpha)}\big]
- \E_{S_t \sim \push{Q}{s}} \big[\ind{S_t \le \quant^\vw(\alpha)}\big] \right] \right| \, d\alpha.
\end{equation*}
Then
\begin{equation*}
\widehat{\Delta}_{w,Q} \;\le\; \sum_{i=1}^n w_i \,\big| F_{\push{\hat{P}^\vw_n}{s}}(s_c^{(i)}) - F_{\push{Q}{s}}(s_c^{(i)}) \big|.
\end{equation*}
\end{proposition}

\begin{proof}
By definition and the tower property,
\begin{equation*}
\widehat{\Delta}_{w,Q}
= \int_0^1 \Big| \E_{\gS_c \sim \push{\hat{P}^\vw_n}{s}} \big[
F_{\push{\hat{P}^\vw_n}{s}}(\quant^\vw(\alpha)) - F_{\push{Q}{s}}(\quant^\vw(\alpha))\big] \Big| \, d\alpha.
\end{equation*}
Applying Jensen's inequality (absolute value is convex),
\begin{equation*}
\widehat{\Delta}_{w,Q}
\le \E_{\gS_c \sim \push{\hat{P}^\vw_n}{s}} \left[\int_0^1
\big| F_{\push{\hat{P}^\vw_n}{s}}(\quant^\vw(\alpha)) - F_{\push{Q}{s}}(\quant^\vw(\alpha)) \big| \, d\alpha \right].
\end{equation*}
Since $\push{\hat{P}^\vw_n}{s}$ is a discrete distribution with atoms $\{s_c^{(i)}\}_{i=1}^n$ and masses $\{w_i\}_{i=1}^n$,
its quantile map $\alpha \mapsto \quant^\vw(\alpha)$ takes the value $s_c^{(i)}$ on an interval of
length exactly $w_i$. Therefore, for any fixed realization of $\gS_c$,
\begin{equation*}
\int_0^1 \big| F_{\push{\hat{P}^\vw_n}{s}}(\quant^\vw(\alpha)) - F_{\push{Q}{s}}(\quant^\vw(\alpha)) \big| \, d\alpha
= \sum_{i=1}^n w_i \, \big| F_{\push{\hat{P}^\vw_n}{s}}(s_c^{(i)}) - F_{\push{Q}{s}}(s_c^{(i)}) \big|.
\end{equation*}
Taking expectation with respect to $\gS_c \sim \push{\hat{P}^\vw_n}{s}$ does not change the right-hand side,
which is deterministic given $\hat{P}^\vw_n$, and the claim follows.
\end{proof}

\subsection{Unlabeled Bound to the Total Coverage Gap}

Before proving Theorem~\ref{theo:unlabeled_tcgap}, we begin by recalling the concept of stochastic dominance, which plays a key role in the argument. Specifically, if $A$ and $B$ are two probability distributions on $\R$, we say $A$ dominates $B$, denoted $A \dominates B$, if $F_{A}(t) \leq F_{B}(t)$ for all $t \in \R$. This notion is especially useful in our context, as will be made clear in the derivations. In particular, stochastic dominance also simplifies the computation of the 1-Wasserstein distance, as captured by the following well-known result.

\begin{lemma}[\citet{de20211}] \label{lemma:stochastic_dominance}
    Let $A$ and $B$ be two probability distributions on $\R$, with $A \dominates B$, then
    \begin{equation*}
        W_1(A, B) = \E_{A}[X] - \E_{B}[X]
    \end{equation*}
\end{lemma}

\begin{reptheorem}{theo:unlabeled_tcgap}
    Let $P$ and $Q$ be two probability measures on $\gX \times \gY$ with $\push{P}{s}$ and $\push{Q}{s}$ their respective pushforward measures by the score function $s: \gX \times \gY \rightarrow \R$. Assume $\push{P}{s}$ is absolutely continuous with respect to the Lebesgue measure with density $p_{\push{P}{s}}(s_c)$. Further let $\push{Q^{\downarrow}_m}{s}$ and $\push{Q^{\uparrow}_m}{s}$ be such that $\push{Q^{\uparrow}_m}{s} \dominates \push{Q}{s} \dominates \push{Q^{\downarrow}_m}{s}$. Then, we have that
    \begin{align}
        \nonumber \Delta_{P, Q} &\leq \frac{1}{2} \int p_{\push{P}{s}}(s_c) \bigg(\left|F_{\push{P}{s}}(s_c) - F_{\push{Q^{\uparrow}_m}{s}}(s_c) \right| + \left|F_{\push{P}{s}}(s_c) - F_{\push{Q^{\downarrow}_m}{s}}(s_c) \right| \\ & \hskip0.525\textwidth + F_{\push{Q^{\downarrow}_m}{s}}(s_c) - F_{\push{Q^{\uparrow}_m}{s}}(s_c)\bigg) ds_c \tag{\ref{eq:unlabeled-density-weighted}} \\
        & \leq \frac{1}{2} \left[\sup_{s_c\in \R} p_{\push{P}{s}}(s_c)\right] \bigg(W_1(\push{P}{s}, \push{ Q^{\uparrow}_m}{s}) + W_1(\push{P}{s}, \push{ Q^{\downarrow}_m}{s}) + \E_{\push{ Q^{\uparrow}_m}{s}}[S] - \E_{\push{Q^{\downarrow}_m}{s}}[S]\bigg) \tag{\ref{eq:unlabeled-w1-bound}}.
    \end{align}
\end{reptheorem}
\begin{proof}[Proof of (\ref{eq:unlabeled-density-weighted})]
    We start from (\ref{eq:density-weighted}) and apply the triangle inequality twice to get
    \begin{align*}
        \Delta_{P,Q} &\leq \int p_{\push{P}{s}}(s_c) \left(\left|F_{\push{P}{s}}(s_c) - F_{\push{Q^{\uparrow}_m}{s}}(s_c) \right| + \left| F_{\push{Q}{s}}(s_c) - F_{\push{Q^{\uparrow}_m}{s}}(s_c) \right| \right) ds_c \\
        \Delta_{P,Q} &\leq \int p_{\push{P}{s}}(s_c) \left(\left|F_{\push{P}{s}}(s_c) - F_{\push{Q^{\downarrow}_m}{s}}(s_c) \right| + \left| F_{\push{Q}{s}}(s_c) - F_{\push{Q^{\downarrow}_m}{s}}(s_c) \right| \right) ds_c
    \end{align*}
    Since all values in these inequalities are non-negative, we can add them up to get
    \begin{align*}
        \Delta_{P,Q} &\leq \frac{1}{2} \int p_{\push{P}{s}}(s_c) \bigg(\left|F_{\push{P}{s}}(s_c) - F_{\push{Q^{\uparrow}_m}{s}}(s_c) \right| + \left|F_{\push{P}{s}}(s_c) - F_{\push{Q^{\downarrow}_m}{s}}(s_c) \right|  \\ &\hskip0.37\textwidth  + \left| F_{\push{Q}{s}}(s_c) - F_{\push{Q^{\uparrow}_m}{s}}(s_c) \right| + \left| F_{\push{Q}{s}}(s_c) - F_{\push{Q^{\downarrow}_m}{s}}(s_c) \right| \bigg) ds_c \\
        &= \frac{1}{2} \int p_{\push{P}{s}}(s_c) \bigg(\left|F_{\push{P}{s}}(s_c) - F_{\push{Q^{\uparrow}_m}{s}}(s_c) \right| + \left|F_{\push{P}{s}}(s_c) - F_{\push{Q^{\downarrow}_m}{s}}(s_c) \right|  \\ &\hskip0.41\textwidth +  F_{\push{Q}{s}}(s_c) - F_{\push{Q^{\uparrow}_m}{s}}(s_c)  + F_{\push{Q^{\downarrow}_m}{s}}(s_c) - F_{\push{Q}{s}}(s_c) \bigg) ds_c  \\
        &= \frac{1}{2} \int p_{\push{P}{s}}(s_c) \bigg(\left|F_{\push{P}{s}}(s_c) - F_{\push{Q^{\uparrow}_m}{s}}(s_c) \right| + \left|F_{\push{P}{s}}(s_c) - F_{\push{Q^{\downarrow}_m}{s}}(s_c) \right| \\ &\hskip0.61\textwidth + F_{\push{Q^{\downarrow}_m}{s}}(s_c) - F_{\push{Q^{\uparrow}_m}{s}}(s_c)\bigg) ds_c,
    \end{align*}    
    where the first equality holds because of the stochastic dominance relationship, which tells us that $F_{\push{Q^{\uparrow}_m}{s}}(t) \leq F_{\push{Q}{s}}(t)$ and $F_{\push{Q}{s}}(t) \leq F_{\push{Q^{\downarrow}_m}{s}}(t)$ for all $t \in \R$.
\end{proof}

\begin{proof}[Proof of (\ref{eq:unlabeled-w1-bound})]
    The result follows from (\ref{eq:unlabeled-density-weighted}) by taking the maximum of $\push{P}{s}$ out of the integral as in the proof of Theorem~\ref{theo:tcgap}. Alternatively, we can also prove (\ref{eq:unlabeled-w1-bound}) directly from (\ref{eq:w1-bound}) in Theorem~\ref{theo:tcgap}. For that, it suffices to show $W_1(\push{P}{s}, \push{Q}{s})$ is upper bounded by the term in parentheses in (\ref{eq:unlabeled-w1-bound}) and the proof immediately follows from Theorem~\ref{theo:tcgap}.
    We start by applying the triangle inequality twice to get
    \begin{align*}
        W_1(\push{P}{s}, \push{Q}{s}) &\leq W_1(\push{P}{s}, \push{Q^{\uparrow}_m}{s}) + W_1(\push{Q}{s}, \push{Q^{\uparrow}_m}{s}) \\
        W_1(\push{P}{s}, \push{Q}{s}) &\leq W_1(\push{P}{s}, \push{Q^{\downarrow}_m}{s}) + W_1(\push{Q}{s}, \push{Q^{\downarrow}_m}{s}).
    \end{align*}
    Since all terms are non-negative, we can sum both inequalities, which gives us
    \begin{multline} \label{eq:sum_inequalities}
        2 W_1(\push{P}{s}, \push{Q}{s}) \leq W_1(\push{P}{s}, \push{Q^{\uparrow}_m}{s}) + W_1(\push{P}{s}, \push{Q^{\downarrow}_m}{s}) +W_1(\push{Q}{s}, \push{Q^{\uparrow}_m}{s}) + W_1(\push{Q}{s}, \push{Q^{\downarrow}_m}{s}).
    \end{multline}
    Using $\push{Q^{\uparrow}_m}{s} \succcurlyeq \push{Q}{s} \succcurlyeq \push{Q^{\downarrow}_m}{s}$, we have from Lemma~\ref{lemma:stochastic_dominance} that
    \begin{align*}
        W_1(\push{Q}{s}, \push{Q^{\uparrow}_m}{s}) = \E_{\push{Q^{\uparrow}_m}{s}}[S] - \E_{\push{Q}{s}}[S] \quad \text{and} \quad W_1(\push{Q}{s}, \push{Q^{\downarrow}_m}{s}) =\E_{\push{Q}{s}}[S] - \E_{\push{Q^{\downarrow}_m}{s}}[S],
    \end{align*}
    and by summing both equalities, the unknown expectations $\E_{\push{Q}{s}}[S]$ cancel out, giving us
    \begin{equation} \label{eq:dominance_applied}
        W_1(\push{Q}{s}, \push{Q^{\uparrow}_m}{s}) + W_1(\push{Q}{s}, \push{Q^{\downarrow}_m}{s}) =  \E_{\push{Q^{\uparrow}_m}{s}}[S] - \E_{\push{Q^{\downarrow}_m}{s}}[S].
    \end{equation}
    Finally, it suffices to plug \eqref{eq:dominance_applied} into \eqref{eq:sum_inequalities} and the proof follows directly from Theorem~\ref{theo:tcgap}.
\end{proof}

\begin{remark}
    It is interesting to note that, if $\push{Q^{\uparrow}_m}{s} \dominates \push{P}{s} \dominates \push{\hat Q^{\downarrow}_m}{s}$, the upper bound simplifies to 
    \begin{equation*}
        \Delta_{P,Q} \leq \int p_{\push{P}{s}}(s_c) \bigg(F_{\push{Q^{\downarrow}_m}{s}}(s_c) - F_{\push{Q^{\uparrow}_m}{s}}(s_c)\bigg) ds_c.
    \end{equation*}
\end{remark}

\subsection{Estimating the Upper Bounds to the Total Coverage Gap from Samples} \label{app:from_samples}
In practice, we often do not have direct access to the distributions themselves and have to rely only on samples. Next, we show how to estimate our bounds from samples by leveraging the Dvoretzky–Kiefer–Wolfowitz (DKW) inequality \cite{dvoretzky1956asymptotic} to get finite-sample guarantees.
We now proceed to show how the upper bounds from Theorem~\ref{theo:tcgap} can be estimated from samples.

\begin{theorem} \label{theo:empirical}
    Let $P$ and $Q$ be two probability measures on $\gX \times \gY$ with $\push{P}{s}$ and $\push{Q}{s}$ their respective pushforward measures by the score function $s: \gX \times \gY \rightarrow \R$. Let $\push{\hat P_n}{s}$ and $\push{\hat Q}{s}$ denote their empirical distributions constructed from $n$ and $m$ samples, respectively. Then, we have with probability at least $1 - 2d$ that
    \begin{align*}
        \Delta_{P,Q} &\leq \int p_{\push{P}{s}}(s_c) \left|F_{\push{\hat P_n}{s}}(s_c) - F_{\push{\hat Q}{s}}(s_c) \right|ds_c + \sqrt{\frac{\log(2/d)}{2n}} + \sqrt{\frac{\log(2/d)}{2m}} \\
        &\leq \Big( \sup_{s_c \in \R}p_{\push{P}{s}}(s_c) \Big) W_1(\push{\hat P_n}{s}, \push{\hat Q}{s}) + \sqrt{\frac{\log(2/d)}{2n}} + \sqrt{\frac{\log(2/d)}{2m}}.
    \end{align*}
\end{theorem}
\begin{proof}
    We start by applying the triangle inequality twice to get
    \begin{align*}
        \Delta_{P,Q} &\leq \int p_{\push{P}{s}}(s_c) \left|F_{\push{P}{s}}(s_c) - F_{\push{\hat Q}{s}}(s_c) \right|ds_c + \int p_{\push{P}{s}}(s_c) \left|F_{\push{Q}{s}}(s_c) - F_{\push{\hat Q}{s}}(s_c) \right|ds_c \\
        &\leq \int p_{\push{P}{s}}(s_c) \left|F_{\push{\hat P_n}{s}}(s_c) - F_{\push{\hat Q}{s}}(s_c) \right|ds_c \\ & \quad + \int p_{\push{P}{s}}(s_c) \left|F_{\push{P}{s}}(s_c) - F_{\push{\hat P_n}{s}}(s_c)\right|ds_c + \int p_{\push{P}{s}}(s_c) \left|F_{\push{Q}{s}}(s_c) - F_{\push{\hat Q}{s}}(s_c) \right| ds_c
    \end{align*}
    From here, we get the weighted CDF version of the bound by applying the DKW inequality to the last two terms. This gives us that, with probability at least $1 - 2d$
    \begin{align*}
        \Delta_{P, Q} & \leq \int p_{\push{P}{s}}(s_c) \left|F_{\push{\hat P_n}{s}}(s_c) - F_{\push{\hat Q}{s}}(s_c) \right|ds_c  \\ &\quad +  \int p_{\push{P}{s}}(s_c) \sqrt{\frac{\log(2/d)}{2n}} ds_c + \int p_{\push{P}{s}}(s_c) \sqrt{\frac{\log(2/d)}{2m}} ds_c \\
         &= \int p_{\push{P}{s}}(s_c) \left|F_{\push{\hat P_n}{s}}(s_c) - F_{\push{\hat Q}{s}}(s_c) \right|ds_c + \sqrt{\frac{\log(2/d)}{2n}} + \sqrt{\frac{\log(2/d)}{2m}}.
    \end{align*}
    The last equality holds since the DKW correction can be pulled outside of the integral and the integrals then sum to one due to $p_{\push{P}{s}}(s_c)$ being a probability density. Finally, once more we can use the fact that $p_{\push{P}{s}}(s_c)$ is non-negative everywhere to get
    \begin{align*}
        \int p_{\push{P}{s}}(s_c) \left|F_{\push{\hat P_n}{s}}(s_c) - F_{\push{\hat Q}{s}}(s_c) \right|ds_c \leq \Big( \sup_{s_c \in \R}p_{\push{P}{s}}(s_c) \Big) W_1(\push{\hat P_n}{s}, \push{\hat Q}{s}),
    \end{align*}
    which gives us the bound expressed in terms of the 1-Wasserstein distance.
\end{proof}

\begin{remark}[Extension to unlabeled bound]
    The upper bounds in Theorem~\ref{theo:unlabeled_tcgap} can also be estimated from samples in a similar manner. In fact, to get the unlabeled version of Theorem~\ref{theo:empirical},
    it suffices to construct two auxiliary empirical distributions such that $\push{\hat Q^{\uparrow}_m}{s} \dominates \push{\hat Q}{s} \dominates \push{\hat Q^{\downarrow}_m}{s}$ and follow the same arguments used to derive Theorem~\ref{theo:unlabeled_tcgap} from Theorem~\ref{theo:tcgap}.
\end{remark}

\begin{remark}
    The DKW inequality only applies to i.i.d.~samples. Therefore, to compute the bound after having optimized the weights in $\push{\hat P_n^\vw}{s}$, we first resample $n_\vw$ samples from this weighted distribution, where $n_\vw = \nicefrac{1}{\sum w_i^2}$ is the effective sample size of $\push{\hat P_n^\vw}{s}$. We then evaluate the bound in Theorem~\ref{theo:empirical} using these new $n_\vw$ samples and replacing $n$ with $n_\vw$.
\end{remark}

\subsection{Restricted Total Coverage Gap} \label{app:restricted_tcgap}
Thus far, we have discussed the total coverage gap, $\Delta_{P,Q}$, which considers miscoverage rates over the full range $[0,1]$ and underpins the main results of this paper, as well as the coverage gap for specific miscoverage rates, $\Delta_{P,Q}(\alpha)$, introduced in Appendix~\ref{app:specific_alpha}. 

In some scenarios, interest may lie in a restricted range of miscoverage rates rather than the entire interval $[0,1]$. To accommodate this, the definition can be extended to a range $[\alpha^{-}, \alpha^{+}]$ with $0 \leq \alpha^{-} \leq \alpha^{+} \leq 1$ such that
\begin{equation}
    \Delta_{P,Q}(\alpha^{-}, \alpha^{+}) := \int_{\alpha^{-}}^{\alpha^{+}} \frac{\Delta_{P,Q}(\alpha)}{\alpha^{+} - \alpha^{-}} d\alpha
\end{equation}

For the restricted coverage above, we have the following result.

\begin{proposition}[Restricted total coverage gap] \label{prop:res_cov}
Let $P$ and $Q$ be probability measures on $\mathcal{X}\times\mathcal{Y}$, and let $s:\mathcal{X}\times\mathcal{Y}\to\R$ be a measurable score function with pushforward measures $s_\# P$ and $s_\# Q$.
Assume $s_\# P$ is absolutely continuous with respect to the Lebesgue measure on $\R$ with density $p_{s_\# P}$ and CDF $F_{s_\# P}$.
For $0\le \alpha^- \le \alpha^+ \le 1$, define
\begin{equation*}
    \Delta_{P,Q}(\alpha^-,\alpha^+)
    \;:=\; \frac{1}{\alpha^+ - \alpha^-}\int_{\alpha^-}^{\alpha^+} 
    \Delta_{P,Q}(\alpha)\, d\alpha,
\end{equation*}
Then $\Delta_{P,Q}(\alpha^-,\alpha^+)$ is upper bounded by
\begin{equation}
\label{eq:A4-weighted-CDF}
    \frac{1}{\alpha^+ -\alpha^-} \int_{\R} p_{s_\# P}(s_c)\; \left| F_{s_\# P}(s_c) - F_{s_\# Q}(s_c) \right|\; \mathbf{1} \left\{ s_c \in \left[ F^{-1}_{s_\# P}(\alpha^-),\,F^{-1}_{s_\# P}(\alpha^+)\right] \right\}\, ds_c .
\end{equation}
\end{proposition}
\begin{proof}
The proof is close to that of Theorem~\ref{theo:tcgap}, following similar steps.
By definition and the Jensen's inequality, we know that

\begin{align*}
\Delta_{P,Q}(\alpha)
&= \biggr| \E_{\gS_c \sim \push{P^n}{s}} \left[ F_{\push{P}{s}}(\quant(\gS_c)) - F_{\push{Q}{s}}(\quant(\gS_c)) \right] \biggr| \\
&\leq \E_{\gS_c \sim \push{P^n}{s}} \left[ \biggr| F_{\push{P}{s}}(\quant(\gS_c)) - F_{\push{Q}{s}}(\quant(\gS_c)) \biggr| \right] 
\end{align*}

Averaging over $\alpha\in[\alpha^-,\alpha^+]$ and applying Fubini to swap the order of integration,
\begin{equation*}
\Delta_{P,Q}(\alpha^-,\alpha^+) \leq
\E_{\gS_c\sim s_\# P} \left[\frac{1}{\alpha^+-\alpha^-} \int_{\alpha^-}^{\alpha^+} \left|F_{s_\# P}(\quant(\gS_c))-F_{s_\# Q}(\quant(\gS_c))\right| \, d\alpha \right]
\end{equation*}

Now fix $\gS_c=\{s_c^{(1)},\dots,s_c^{(n)}\}$ sorted increasingly. Over the full range $[0,1]$, each calibration score occupies an interval of length $1/n$ in the quantile map. Restricting to $[\alpha^-,\alpha^+]$ simply zeroes out scores whose CDF lies outside this range. Thus
\begin{multline*}
    \int_{\alpha^-}^{\alpha^+}\left|F_{s_\#P}(\quant(\gS_c))-F_{s_\#Q}(\quant(\gS_c))\right|\,d\alpha
    = \\ \sum_{i=1}^n \frac{1}{n}\,\ind{F_{s_\#P}(s_c^{(i)})\in[\alpha^-,\alpha^+]}\, \left| F_{s_\#P}(s_c^{(i)})-F_{s_\#Q}(s_c^{(i)}) \right|.
\end{multline*}

Note that this is the same argument of (\ref{eq:empirical}) in the proof of Theorem~\ref{theo:tcgap}, but here we take extra care to restrict the range to $[\alpha^-,\alpha^+]$. Plugging back we have
\begin{equation*}
\Delta_{P,Q}(\alpha^-,\alpha^+) \leq
\E_{\gS_c\sim s_\# P} \left[\frac{1}{n}\sum_{j=1}^n
\ind{F_{s_\#P}(s_c^{(j)})\in[\alpha^-,\alpha^+]} \, \left|F_{s_\#P}(s_c^{(j)})-F_{s_\#Q}(s_c^{(j)})\right| \right].
\end{equation*}

Since the calibration data point is identically distributed, the expectation of the sample mean equals the population mean:
\begin{equation*}
\E_{\gS_c\sim s_\# P} \Bigg[\frac{1}{n}\sum_{j=1}^n h(s_c^{(j)})\Bigg] = \E_{\gS_c\sim s_\# P} [h(S)],
\end{equation*}
with $h(s)=\ind{F_{s_\#P}(s_c^{(i)})\in[\alpha^-,\alpha^+]}\, \left| F_{s_\#P}(s_c^{(i)})-F_{s_\#Q}(s_c^{(i)}) \right|$. Therefore,
\begin{equation*}
\Delta_{P,Q}(\alpha^-,\alpha^+) \le
\frac{1}{\alpha^+-\alpha^-}\int_{\R} p_{s_\#P}(s)\,\left|F_{s_\#P}(s)-F_{s_\#Q}(s)\right|\,
\ind{F_{s_\#P}(s_c^{(i)})\in[\alpha^-,\alpha^+]} \,ds,
\end{equation*}
which is the desired bound.
\end{proof}

\begin{remark}[On Wasserstein relaxations] It is also possible to connect the result of Proposition~\ref{prop:res_cov} to the 1-Wasserstein distance. Since $\ind{\cdot}\le 1$, a loose relaxation of (\ref{eq:A4-weighted-CDF}) gives
\begin{align*}
    \Delta_{P,Q}(\alpha^-,\alpha^+) & \leq \frac{1}{\alpha^+-\alpha^-} \Big( \sup_{s_c \in \R}p_{\push{P}{s}}(s_c) \Big) \, \int_{\R} \big| F_{s_\# P}(s)-F_{s_\# Q}(s) \big|\, ds \\ &= \frac{1}{\alpha^+-\alpha^-} \Big( \sup_{s_c \in \R}p_{\push{P}{s}}(s_c) \Big)\, W_1(s_\# P, s_\# Q),
\end{align*}
which may be overly conservative in practice.
\end{remark}

While one could optimize the upper bound in Proposition~\ref{prop:res_cov} directly, preliminary experiments show that this approach yields only marginal improvements in coverage. We see two likely reasons for this. First, optimizing a bound restricted to a specific coverage range may be inherently more challenging; for instance, the pointwise bound for a fixed $\alpha$ in Appendix~\ref{app:specific_alpha} also failed to deliver better empirical performance. Second, the total coverage gap already provides a strong and well-behaved objective, leaving little room for alternative formulations to offer significant gains. Nevertheless, these more targeted objectives remain an interesting direction for future work, particularly in applications where coverage guarantees over a narrow range of $\alpha$ are critical.

\subsection{Upper Bound to the Coverage Gap for a Specific Target Miscoverage Rate} \label{app:specific_alpha}
As discussed in the main paper, similar techniques can also be employed to derive an upper bound on $\Delta_{P,Q}(\alpha)$, the coverage gap corresponding to a given miscoverage rate $\alpha$. This result is formalized in Theorem~\ref{theo:specific_alpha_gap}, which also outlines how it can be estimated from samples using the DKW inequality.
\begin{theorem}\label{theo:specific_alpha_gap}
    Let $P$ and $Q$ be two probability measures on $\gX \times \gY$ with $\push{P}{s}$ and $\push{Q}{s}$ their respective pushforward measures by the score function $s: \gX \times \gY \rightarrow \R$. Let $\push{\hat P_n}{s}$ and $\push{\hat Q_m}{s}$ denote their empirical distributions constructed from $n$ and $m$ samples, respectively. 
    Further, let $\push{\hat Q^{\downarrow}_m}{s}$ and $\push{\hat Q^{\uparrow}_m}{s}$ be such that $\push{\hat Q^{\uparrow}_m}{s} \dominates \push{\hat Q_m}{s} \dominates \push{\hat Q^{\downarrow}_m}{s}$. Then, we have with probability at least $1-2d$ that
    \begin{multline*}
        \Delta_{P,Q}(\alpha) \leq \frac{1}{2} \E_{\gS_c \sim \push{P^n}{s}} \biggr[ \bigg| F_{\push{\hat P_n}{s}}(\quant(\gS_c)) - F_{\push{\hat Q^{\downarrow}_m}{s}}(\quant(\gS_c)) \bigg| \\ + \bigg| F_{\push{\hat P_n}{s}}(\quant(\gS_c)) - F_{\push{\hat Q^{\uparrow}_m}{s}}(\quant(\gS_c)) \bigg| + \bigg| F_{\push{\hat Q^{\downarrow}_m}{s}}(\quant(\gS_c)) - F_{\push{\hat Q^{\uparrow}_m}{s}}(\quant(\gS_c)) \bigg| \biggr] \\ + \sqrt{\frac{\log(2/d)}{2n}} + \sqrt{\frac{\log(2/d)}{2m}}
    \end{multline*}
\end{theorem}
\begin{proof}
    Per definition the coverage gap for a specific target miscoverage rate $\alpha$ is given by 
    \begin{align*} 
        \Delta_{P,Q}(\alpha) &:= \big| P(S_t \leq \quant(S_c)) - Q(S_t \leq \quant(S_c)) \big |  \\
        \nonumber &= \biggr|  \E_{S_t \sim \push{P}{s}} \left[ \E_{\gS_c \sim \push{P^n}{s}}  \left[\ind{S_t \leq \quant(S_c)}\right] \right]    \\ 
        & \hspace{5cm} - \E_{S_t \sim \push{Q}{s}} \left[ \E_{\gS_c \sim \push{P^n}{s}}  \left[\ind{S_t \leq \quant(S_c)}\right] \right] \biggr|, \\
        &= \biggr| \E_{\gS_c \sim \push{P^n}{s}} \left[ F_{\push{P}{s}}(\quant(\gS_c)) - F_{\push{Q}{s}}(\quant(\gS_c)) \right] \biggr| \\
        &\leq \E_{\gS_c \sim \push{P^n}{s}} \left[ \biggr| F_{\push{P}{s}}(\quant(\gS_c)) - F_{\push{Q}{s}}(\quant(\gS_c)) \biggr| \right] 
\end{align*}
where the last inequality follows from Jensen's inequality. At this point we introduce the empirical distributions $\push{\hat P_n}{s}$ and $\push{\hat Q_m}{s}$ by applying the triangle inequality twice.
\begin{align*}
    \Delta_{P,Q}(\alpha) &\leq \E_{\gS_c \sim \push{P^n}{s}} \biggr[ \bigg| F_{\push{\hat P_n}{s}}(\quant(\gS_c)) - F_{\push{Q}{s}}(\quant(\gS_c)) \bigg| \\
    & \hspace{6cm} + \bigg| F_{\push{P}{s}}(\quant(\gS_c)) - F_{\push{\hat P_n}{s}}(\quant(\gS_c)) \bigg| \biggr] \\
    &\leq \E_{\gS_c \sim \push{P^n}{s}} \biggr[ \bigg| F_{\push{\hat P_n}{s}}(\quant(\gS_c)) - F_{\push{\hat Q_m}{s}}(\quant(\gS_c)) \bigg| \\ 
    & \hspace{0.5cm} + \bigg| F_{\push{P}{s}}(\quant(\gS_c)) - F_{\push{\hat P_n}{s}}(\quant(\gS_c)) \bigg| + \bigg| F_{\push{Q}{s}}(\quant(\gS_c)) - F_{\push{\hat Q_m}{s}}(\quant(\gS_c)) \bigg|\biggr]
\end{align*}
We then apply the DKW inequality to get with probability $1-2d$
\begin{align*}
    \Delta_{P,Q}(\alpha) &\leq \E_{\gS_c \sim \push{P^n}{s}} \biggr[ \bigg| F_{\push{\hat P_n}{s}}(\quant(\gS_c)) - F_{\push{\hat Q_m}{s}}(\quant(\gS_c)) \bigg|\biggr] + \sqrt{\frac{\log(2/d)}{2n}} + \sqrt{\frac{\log(2/d)}{2m}}.
\end{align*}
Finally, we introduce the two auxiliary distributions $\push{Q^{\uparrow}_m}{s} \dominates \push{Q}{s} \dominates \push{Q^{\downarrow}_m}{s}$ by once more applying the triangle inequality twice and summing the inequalities to get the final result.
\end{proof}

Unfortunately, preliminary experiments indicate that the bound presented in Theorem~\ref{theo:specific_alpha_gap} is loose and of limited practical utility, unless the auxiliary distributions are close to the true score distribution under $Q$. Further research is required to derive meaningful bounds for specific values of $\alpha$ in the absence of labeled data. Nevertheless, if the goal is to enhance coverage for a particular $\alpha$, one can directly optimize the bound in Theorem~\ref{theo:specific_alpha_gap}. Following prior work \cite{bellotti2021optimized,stutz2022learning,correia2024information}, this can be achieved by introducing differentiable relaxations in the computation of quantiles and empirical CDFs, thereby enabling gradient-based optimization with respect to the weights in $\push{\hat P_n^\vw}{s}$. However, this approach has proven less effective than optimizing upper bounds on the total coverage gap.

\newpage
\section{Applicability, Limitations, and Extensions}

\paragraph{Prior-Knowledge-Based Sandwiching Design.}
Our bounds in Theorem~\ref{theo:unlabeled_tcgap} rely on auxiliary distributions $(\push{Q_{\downarrow}}{s}, \push{Q_{\uparrow}}{s})$ that stochastically dominate the unknown test score distribution $\push{Q}{s}$. In the absence of prior knowledge, we use uninformed constructions such as $(\min,\max)$ or $(f,U)$, which work well empirically but may yield loose bounds. A natural extension is to exploit domain-specific or structural prior information to design tighter sandwiching distributions.

For example, if labels are organized in a hierarchy (e.g., ImageNet superclasses), side information can constrain the feasible set of candidate labels for each test point. This allows constructing $\push{Q_{\downarrow}}{s}$ and $\push{Q_{\uparrow}}{s}$ by selecting the most plausible and least plausible labels within that subset, leading to provably tighter bounds. Similarly, if the distribution shift is known to be bounded (e.g., perturbations within an $\ell_\infty$ ball), optimization procedures can identify feasible score ranges that respect these constraints. The smaller the feasible set, the closer the auxiliary distributions approximate $\push{Q}{s}$, improving both theoretical guarantees and empirical performance.

Exploring these strategies (hierarchical constraints, perturbation models, or other structured priors) represents a promising direction for future work, as it bridges the gap between general-purpose bounds and application-specific robustness.

\paragraph{Robustness to Misspecification.}
Our bounds in Theorem~3.3 assume that the auxiliary distributions $(\push{Q_{\downarrow}}{s}, \push{Q_{\uparrow}}{s})$ satisfy the stochastic dominance relationship $\push{Q_{\uparrow}}{s} \succeq \push{Q}{s} \succeq \push{Q_{\downarrow}}{s}$. This condition is guaranteed for the $(\min,\max)$ construction and was observed to hold in many experiments for $(f,U)$, which uses scores derived from the model’s predicted distribution $Q_f(Y|X)$ and from a uniform distribution over labels. However, $(f,U)$ does not always satisfy this assumption.

Coverage improvements in these cases can be explained by the fact that the optimization objective remains effective whenever the auxiliary distributions help move the calibration score distribution closer to the test score distribution. This alignment, even if imperfect, can still reduce the coverage gap. However, this is an important caveat: if the auxiliary distributions fail to capture the nature of the shift, optimization may bias the calibration distribution in the wrong direction and worsen coverage. Handling this risk requires care.

Future work should systematically study these failure modes and develop safeguards. Promising directions include diagnostics to detect dominance violations, adaptive refinement of auxiliary distributions based on empirical checks, and regularization strategies to prevent extreme deviations when auxiliary distributions are poorly aligned with the test distribution.

\paragraph{Ambiguous Ground Truth.}
Ambiguous ground truth arises in settings where each instance may correspond to multiple plausible labels with associated probabilities, such as in fine-grained classification or scenarios with inherent uncertainty. This problem has recently attracted attention in the conformal prediction literature \cite{stutz2023conformal,caprio2025conformalized}. 

Our method operates directly on nonconformity scores, which are typically unidimensional, without imposing any assumptions on how the scores are constructed. This property makes it naturally compatible with most CP techniques, including scenarios involving ambiguous ground truth. For example, following \citet{stutz2023conformal}, one can define a score function as a weighted average of class-specific scores under a plausibility vector $\lambda \in \Delta_K$, i.e.,
\begin{equation*}
    s'(x,\lambda) := \sum_{k=1}^K \lambda_k \, s(x,y_k).
\end{equation*}

Once such a score is defined, our approach can learn weights over the calibration data and compute a threshold that adapts under distribution shift, just as in the standard setting.

The main challenge lies in constructing auxiliary distributions for ambiguous ground truth. Instead of working with a discrete set of labels, we must consider distributions over the simplex, which complicates the design of $(\push{Q_{\downarrow}}{s}, \push{Q_{\uparrow}}{s})$. While a min–max construction remains possible, it may be overly conservative, as the resulting auxiliary distributions could be too far apart to yield tight bounds. Future work could explore more informative strategies for building auxiliary distributions in this setting, potentially leveraging prior knowledge or structural constraints on the plausibility vectors.
\newpage
\section{Extra experimental details and results} \label{app:experiments}
In this section, we present additional details about our experimental setup and supplementary results. We begin by outlining how our methods fit within the split conformal prediction framework in Section~\ref{app:split-cp} and specifically in Algorithm~\ref{alg:ours-end2end}. Next, we detail the baseline methods in Section~\ref{app:baselines}, followed by a description of the datasets used in Section~\ref{app:datasets}. Finally, in Section~\ref{app:design}, we discuss key design choices and ablation studies that may offer valuable insights for future research.

Before proceeding, we comment on a few technical details. We note that the code was implemented in Python 3 using PyTorch \citep{paszke2017automatic} and all experiments were conducted on a single commercial NVIDIA GPU with 12 GB of memory.

\bigskip
\begin{algorithm}[H]
   \caption{End-to-end Non-exchangeable CP with Optimal Transport (Split CP)}
   \label{alg:ours-end2end}
\begin{algorithmic}
   \STATE {\bfseries Input}:
   \STATE \hspace{12pt} $n$ labeled samples $\{(x_i, y_i)\}_{i=1}^n$ from $P$
   \STATE \hspace{12pt} $m$ unlabeled samples $\{x_j\}_{j=n+1}^{n+m}$ from $Q$
   \STATE \hspace{12pt} target miscoverage $\alpha \in (0,1)$
   \STATE \hspace{12pt} score function $s$
   \STATE Initialize unnormalized weights $\tilde \vw = \{\tilde w_i\}_{i=1}^n$ or weight function $w_\theta$
   \STATE Compute scores $\{s(x_i, y_i)\}_{i=1}^n$
   \STATE Compute score vectors $\{\vs(x_j)\}_{j=n+1}^{n+m}$ \hfill {\small // $\vs(x)=\{s(x,y):y \in \gY\}$}
   \REPEAT
      \STATE Construct $\push{\hat Q^{\downarrow}}{s}$ and $\push{\hat Q^{\uparrow}}{s}$ from $\{\vs(x_j)\}_{j=n+1}^{n+m}$ \hfill {\small // e.g., $(\min,\max)$ or $(f,U)$}
      \STATE Fit KDE to $\{s(x_i, y_i)\}_{i=1}^n$ with weights $\vw$
      \STATE Update $\tilde \vw$ or $w_\theta$ to minimize either (\ref{eq:unlabeled-density-weighted}) or (\ref{eq:unlabeled-w1-bound}) \hfill {\small // weighted-CDF or 1-Wasserstein bound}
   \UNTIL{convergence or max steps}
   \STATE {\bfseries Weighted normalization}: Compute normalized weights $\vw$ \hfill {\small // accounting for weight of test point}
   \STATE {\bfseries Weighted threshold}: $q_{1-\alpha} \leftarrow Q^{\vw}_\alpha\big(\{s(x_i, y_i)\}_{i=1}^n\big)$
   \STATE {\bfseries Prediction sets}: for each $x_t$, set $C(x_t) \leftarrow \{\,y \in \gY \;:\; s(x_t,y) \le q_{1-\alpha}\,\}$
   \STATE {\bfseries Output}: $\{C(x_t)\}_{t=1}^T$ (if $\{x_t\}$ provided) and learned weights $\vw$
\end{algorithmic}
\end{algorithm}

\subsection{Split Conformal Prediction Procedure}
\label{app:split-cp}

We follow a standard split conformal prediction framework, with the main difference being that we introduce weights over the calibration nonconformity scores to better align their empirical distribution with that of test scores. Concretely, we construct prediction sets by thresholding on the nonconformity scores given by one minus the model-assigned probabilities. However, our methods are agnostic to the choice of score function and could be applied to other approaches like APS \cite{romano2020classification}. We outline the complete split CP algorithm we use, including  weight optimization with our bounds, in Algorithm~\ref{alg:ours-end2end}.

\subsubsection{Weighting of Test Samples}
In standard split conformal prediction, a test data point is implicitly assigned a weight of $\nicefrac{1}{n+1}$, preserving symmetry with the calibration set. Extensions to non-exchangeable settings like \cite{tibshirani_conformal_2019,barber2023conformal} also address this issue explicitly. \citet{tibshirani_conformal_2019} propose assigning the test point a weight proportional to its likelihood ratio under the shifted distribution (see details in Section~\ref{sec:lr}), while \citet{barber2023conformal} fix the unnormalized weight of the test point to 1, reflecting the fact that it already comes from the target distribution.

\paragraph{Free-form weights} We adopt a similar convention of \cite{barber2023conformal} for free-form weights, assigning test points the unnormalized weight $\tilde w_{n+1} = 1$. Intuitively, importance weights are meant to correct for distribution mismatch, and the test point is already drawn from the target distribution, which justifies unit weights. In that case, we have the following weighted empirical distribution $\push{P_n^\vw}{s}$ 
\begin{equation*}
    \push{\hat P^\vw_n}{s} = \sum_{i=1}^n w_i \delta_{s(x_i, y_i)} \quad \text{with normalized weights} \quad w_i = \frac{\tilde w_i}{1 + \sum_{j=1}^{n} \tilde w_j}.
\end{equation*}

\paragraph{Weight Function} When using a weight function, we adopt a similar strategy to that of \citet{tibshirani_conformal_2019}. Unlike their setting, where the weight function maps inputs $x$ to weights, our function maps nonconformity scores to weights. Consequently, for each candidate label $y \in \gY$, we obtain a distinct test weight $w_\theta(s(x_{n+1}, y))$. Computing all these weights can be expensive when $|\gY|$ is large. To mitigate this cost, we approximate the test weight by taking the maximum over the score range $[s_{\min}, s_{\max}]$ to get a conservative upper bound for the conformal threshold:
\begin{equation*}
\tilde w_{n+1} = \max_{s \in [s_{\min}, s_{\max}]} w_\theta(s).
\end{equation*}
In our classification experiments, we define nonconformity scores as one minus the probabilities assigned by the model, with $[s_{\min}, s_{\max}]=[0,1]$. For calibration points, we compute unnormalized weights directly as $\tilde w_i = w_\theta(s(x_i, y_i))$ for $i \in \{0, \ldots, n\}$. Finally, the weighted empirical distribution $\push{P_n^\vw}{s}$ is given by
\begin{equation*}
    \push{\hat P^\vw_n}{s} = \sum_{i=1}^n w_i \delta_{s(x_i, y_i)} \quad \text{with normalized weights} \quad w_i = \frac{\tilde w_i}{\sum_{j=1}^{n+1} \tilde w_j}.
\end{equation*}

\subsection{Other Implementation Details}

In all experiments, free-form weights are randomly initialized from a uniform distribution in $[0, 1]$ but mapped to log space for stability. When learning a weight function, we implement it as a small multi-layer perceptron (MLP) applied directly to scalars representing nonconformity scores. In all cases, the architecture is a simple MLP with shape $1 \to 256 \to 16 \to 8 \to 1$ and ReLU activations followed a tempered $\tanh$ output to bound log-weights in $[-20,20]$. Optimization follows the exact same procedure for both parametrizations. In particular, for the image classification tasks, we use Adam with learning rate $10^{-3}$ for all datasets, varying the number of steps from $1000$ to $5000$ steps depending on shift severity; see Section~\ref{app:datasets} for exact details. 

\subsection{Baselines} \label{app:baselines}
\subsubsection{``Oracle''}
We use the term ``oracle'' to describe the marginal coverage and expected prediction set size achieved when both calibration and testing are performed on samples from $Q$. This setup guarantees the desired coverage and, in our context where the test scores remain fixed, represents the best possible outcome. We only show these results as a reference for what we would get if we knew the true $\push{Q}{s}$, highlighting that the prediction set sizes are considerably larger only because the model is less accurate, as it was trained on samples from $P$ and not $Q$.

\subsubsection{Likelihood Ratios} \label{sec:lr}
\citet{tibshirani_conformal_2019} also proposed to reweight calibration points from $P$ and apply split CP using a weighted distribution of calibration scores $\push{\hat P_n^\vw}{s} = \frac{1}{n}\sum_{i=1}^n w_i \delta_{s(x_i, y_i)}$. Since they only address covariate shifts, it is easy to show the optimal unnormalized weights are given by likelihood ratios of the form $\nicefrac{dQ(x_i)}{dP(x_i)}$. Unfortunately, learning accurate likelihood ratios is known to be challenging, especially when the two distributions are far apart. In our experiments, we applied the telescopic density ratio estimation approach of \citep{rhodes2020telescoping}, which we found useful in the context of more severe shifts. In all cases, learning likelihood ratios directly on the input space $\gX$ proved challenging and we found more success when operating on the space of scores, i.e., fitting the density ratio estimator to map vector of scores $\vs(x_i)$ to (approximate) ratios $\tilde w_i \approx \nicefrac{dQ(x_i)}{dP(x_i)}$. The weights $w_i$ are then recovered by normalizing the likelihood ratios over the calibration set and the test point in question. That is, at test time, we must first evaluate the density ratio estimator to get $\tilde w_{n+1} \approx \nicefrac{dQ(x_{n+1})}{dP(x_{n+1})}$ and then compute normalized weights as
\begin{equation*}
    w_i = \frac{\tilde w_i}{\sum_{j=1}^{n+1} \tilde w_j}.
\end{equation*}

For the large datasets, ImageNet-C and iWildcam, our density ratio estimator was given by a neural network with two hidden layers and ten bridges. In this context, each bridge predicts the density ratio between two intermediary distributions defined by a mixture of samples from $P$ and $Q$. In our experiments, we constructed the intermediary distributions via linear combinations as detailed in \citep{rhodes2020telescoping}. We train the density ratio estimator for ten epochs with a learning rate of $1e^{-3}$ and weight decay of $1e^{-3}$ to avoid overfitting. In the regression task, we also used an MLP with two hidden layers and the same learning rate of $1e^{-3}$ but forewent the telescopic approach (no bridges) as it did not prove useful. Regarding the architectures, we used ReLU activations and kept the hidden size constant and equal to the input size, i.e., 1000 for ImageNet-C, 182 for iWildCam, and 4 for the regression task.

\subsubsection{Entropy scaled Conformal Prediction}
We also compare our methods to Entropy scaled Conformal Prediction (ECP) proposed by \citep{kasa2024adapting}, which constructs prediction sets as 
\begin{align*}
    u_{\gD_Q^{(2)}} &= \quant\left(\{h(\vf(X_i))\}_{i=1}^n\right) \\
    \gC_{Kasa}(X_t) &= \left\{y \in \gY : s(X_t, y) \cdot \max(1, u_{\gD_Q^{(2)}}) \leq \quant\left(\gS_c \right) \right\}, 
\end{align*}
where $\vf(X_i) =\{f(X_i, y') : y' \in \gY \}$ is the set of probabilities assigned to each class by the underlying predictor $f$, and $h(\vf(X_i)) = -\sum_{y \in \gY} f(X_i, y)\log(f(X_i, y))$ is the entropy of this set of probabilities.
Their method is designed to work with test-time adaption methods \citep{ttt, liu2021ttt++}, which adapt the model $f$ in a stream of test samples. However, it can also be applied to our setting where the model is kept fixed and a set of test samples $\gD_Q^{(2)}$ is observed all at once, as formulated above.

\subsubsection{Conformal Prediction With Conditional Guarantees}
\citet{gibbs2023conformal} proposed a method that, under a prespecified function class of covariate shifts, guarantees conditional coverage, i.e., ensuring the prediction set contains the true label for every test point $X_t$
\begin{equation*}
    \P(Y_t \in \gC(X_t) | X_t) \geq 1- \alpha.
\end{equation*}
Their approach essentially changes the conformal threshold for each new test point by learning a function $\hat g_{s(X_t, y)}$ as follows
\begin{align} \label{eq:gibbs}
    \nonumber \hat g_{S} & = \argmin_{g \in \gF} \frac{1}{n+1} \sum_{i=1}^n \ell_\alpha(g(X_i), S_i) + \frac{1}{n+1}\ell_\alpha(g(X_t), S) \\
    \gC_{Gibbs}(X_t) &= \left\{y \in \gY : s(X_t, y) \leq \hat g_{s(X_t, y)}(X_t) \right\}, 
\end{align}
where $\gF$ is the function class of distribution shifts of interest, $\ell_\alpha$ is the pinball loss with target quantile level $\alpha$, and $S \in \R$ refers to the unknown nonconformity score of $X_t$. By computing a new threshold per test point and providing conditional guarantees, their method is inherently more powerful than ours. However, this extra power comes with extra limitations:
\begin{itemize}
    \item \textbf{Function class $\gF$ is typically unknown}. Conditional guarantees are impossible in the most general case of an arbitrary infinite dimensional class \citep{vovk2012conditional,foygel2021limits}. Therefore, we must constrain ourselves to a prespecified class of functions, which requires precise knowledge about the types of distribution shift we expect in practice. In experiments with large datasets like ImageNet-C or iWildCam, it is not clear how to define $\gF$ effectively.
    \item \textbf{Computational cost at test time}. \citet{gibbs2023conformal} propose an efficient algorithm to optimize (\ref{eq:gibbs}) that leverages the monotonicity of quantile regression to avoid evaluating each possible test score for $X_t$. Yet, this procedure still significantly increases latency at test time, and in our hardware, it took 1.5 seconds per test point when using 300 calibration points, and 30 seconds when using 1000 calibration points. Due to this extra computational cost, we only used 300 calibration points when applying the method of \citep{gibbs2023conformal}.
    \item \textbf{Prediction set size}. In practice, we observed that prediction sets produced by the method of \citep{gibbs2023conformal} to be significantly larger than optimal. This is in part due to the stronger conditional guarantee but might reduce the usefulness of the prediction sets in practice.
\end{itemize}

In our experiments, we used the official implementation available at \href{https://github.com/jjcherian/conditional-conformal}{github.com/jjcherian/conditional-conformal}. Similarly to the likelihood ratio baseline described above, we applied their method to the vector of scores instead of the input space. Since the function class corresponding to the type of distribution shifts observed in ImageNet-C and iWildCam are hard to define in practice, we applied the most general approach using radial basis function (RBF) kernels with hyperparameters $\gamma=12.5$ and $\lambda=0.005$ (see the official implementation for details) for ImageNet-C and iWildCam experiments, since we found this to work best in preliminary experiments.

\subsection{Datasets} \label{app:datasets}
\subsubsection{Regression - Synthetic Data}
For the toy regression task we adopt a setting similar to the one proposed in \citep{yang2024doubly}, where we have a regression problem with 4-dimensional input variable $X \sim \N(0, I_4)$, and target variable given by $Y= 210 + 27.4 X_1 + 13.7 X_2 + 13.7 X_3 + 13.7 X_4 + \epsilon$, with $\epsilon \sim \N(0, 1)$. We will refer to this distribution on $\gX \times \gY$ as the unshifted distribution $P$. We then induce a covariate shift to get a new distribution $Q$ via exponential tilting by resampling the data with weights $w_{\text{tilt}}(x) = \exp(-1 x_1 + 0.5 x_2 - 0.25 x_3 - 0.1 x_4).$ This automatically gives us ground-truth likelihood ratios, since $w_{\text{tilt}}(x) = \nicefrac{dQ(x)}{dP(x)}$ by design. Note that these tilting weights should still be normalized over the calibration set to define a proper empirical distribution we can apply conformal prediction over. 

For the experiment depicted in Figure~\ref{fig:regression}, we repeat 500 simulations, each time sampling new datasets $\gD_P^{(1)}$, $\gD_P^{(2)}$, $\gD_Q^{(1)}$, and $\gD_Q^{(2)}$, each with 1000 samples. We use the regression-as-classification method of \citet{guhaconformal}, splitting the output space into 50 equally spaced bins. For the predictor, we train a multilayer perceptron (MLP) with a single hidden layer of 256 units and ReLU activations. For the density ratio estimator, we use an MLP with one hidden layer with 32 neurons, which we train to distinguish samples from $\gD_P^{(2)}$ and $\gD_Q^{(1)}$ by minimizing the logistic loss as common in density estimation tasks. For our methods, we only considered the free-form parametrization for this experiment. We repeat the optimization process in Algorithm~\ref{alg:ours} for 10 steps, each time updating all the weights using Adam \citep{kingma2014adam} with a learning rate of $0.1$.

\subsubsection{ImageNet-C}
In all cases, the underlying classifier is the pretrained ResNet-50 from torchvision \citep{torchvision2016}. To simulate label shift, as in previous work \citep{garg2023rlsbench}, we do so by resampling data points (without replacement) according to a new label marginal $Q(Y) \sim \text{Dir}(\vc)$ where the concentration parameters $\vc$ are given by $c_k = P(Y=k) * \gamma$. The parameter $\gamma$ controls the intensity of the label shift, with lower values of $\gamma$ producing more skewed label marginals. However, since we sample without replacement, low values of $\gamma$ lead to small sample sizes ($<50$) and, consequently, noisy results. For that reason, we set $\gamma=10$ which yields a significant label shift while producing sample sizes of around 300 samples.

ImageNet-C \citep{imagenetc} comprises 15 different types of corruption applied on top of the original validation dataset of ImageNet \citep{russakovsky2015imagenet}, which contains 50K samples. We repeat each experiment with 10 random seeds. As explained in Section~\ref{sec:exp_reg}, each time we randomly split the corrupted data $\gD_Q$ into two: $\gD_Q^{(1)}$ used to learn the weights or density ratio model, and $\gD_Q^{(2)}$ for testing. Since we use a pretrained model as classifier, we do not need to reserve a subset of the clean data $\gD_P$ (the original validation dataset) for training the model. However, one thing to note is that in ImageNet-C $\gD_Q$ is constructed using the same images in $\gD_P$. Thus, we have to ensure the calibration dataset, which we denote $\gD_P^{(2)}$ to be consistent with Section~\ref{sec:exp_reg}, does not contain any of the images in either $\gD_Q^{(1)}$ or $\gD_Q^{(2)}$. This will affect the sizes of each of these sets, as described in the following sections. In all ImageNet-C experiments, unless explicitly stated otherwise, we construct calibration and test sets with the following sizes: calibration sets with $|\gD_P^{(2)}|=300$ and $|\gD_Q^{(1)}|=300$, and test sets with $|\gD_Q^{(2)}|=30000$.

The underlying classifier was a pretrained ResNet-50 available in Torchvision package \cite{torchvision2016}, which was kept fixed in all experiments. We only learn weighting scheme for the calibration data points, and the model as well as the nonconformity score function remain unchanged in all cases. In each optimization step and for both parametrizations, we backprop through all weights (no batching) using Adam with a learning rate of $1\mathrm{e}{-3}$ and $\beta = (0.9, 0.999)$. These hyperparameters were the same across all runs, with only the number of optimization steps varying: $1000$ steps for distribution shifts of severity 1, $3000$ steps for severity 3, and $5000$ steps for severity 5. We observed our methods to be fairly robust to these hyperparameters. One should only keep in mind that the more severe the shift, the longer the optimization or the larger the learning rate should be, as demonstrated in our approach.

In Table~\ref{tab:imagenetc_cov_big}, we present the empirical coverage obtained for each method and corruption in ImageNet-C, giving a more complete picture of the ImageNet-C results reported in Table~\ref{tab:imageclassification}.

\subsubsection{iWildCam} \label{app:influence_number_of_samples}
iWildCam involves images of animals from different camera traps that aim to monitor biodiversity loss. The distribution shift arises from the differences in the characteristics of the environment of each camera trap (e.g., changes in illumination, camera angle, background, etc.).
We use different subsets of camera traps for training, validation and testing, which induces a distribution shift. For the classifier, we train a ResNet-50 model from scratch on the training set. As for ImageNet-C, we repeat each experiment with 10 random seeds.

We experiment with two different settings of distribution shift: the natural one (i.e., differences in camera-traps between validation and testing) that already exists in the data, and a combination of the natural shift with a label shift induced by changing the marginal over the labels via a Dirichlet distribution in the same way as for the ImageNet-C dataset, also with $\gamma=10$. For this experiment, we also used Adam with a learning rate of $1\mathrm{e}{-3}$ and $\beta = (0.9, 0.999)$, running optimization for $1000$ steps. As for the ImageNet-C dataset, we use calibration sets with $|\gD_P^{(2)}|=300$ and $|\gD_Q^{(1)}|=300$ for the experiments reported in the main paper, with the test set composed of 10000 samples, i.e., $|\gD_Q^{(2)}|=10000$. However, we also use the iWildCam dataset to study the effect of variations in sample sizes in our methods, as explained in the next section.

\subsection{Further Discussion} \label{app:design}
In this section, we examine key design choices, including the impact of sample size and the performance differences arising from the selection of bound and auxiliary distributions. For clarity, the empirical results presented in this section focus exclusively on the free-form parametrization, i.e., we directly optimize the weights over the calibration scores.

\subsubsection{Influence of the number of samples from calibration and test distributions}
In Tables~\ref{tab:iwildcam_total_min_max} and~\ref{tab:iwildcam_total_random_model}, we illustrate how the final total coverage gap $\Delta_{P,Q}$ varies with $|\gD_P^{(2)}|$, the number of calibration samples from $P$, and $|\gD_Q^{(1)}|$, the number of unlabeled samples from $Q$. As anticipated, the method's performance improves with an increase in the number of available samples. Interestingly, the number of calibration samples from $P$ appears to be more crucial for performance. This is encouraging, as it suggests that collecting or waiting for a large number of unlabeled samples from the test distribution is unnecessary, with no significant gains observed beyond 1000 samples.

We observed a significant reduction in the total coverage gap in all cases where the number of calibration samples from $P$ was 100 or more. For smaller calibration samples, the observed change in coverage was minimal or even slightly detrimental, as in the case of $(\push{\hat Q_m^{\min}}{s}$, $\push{\hat Q_m^{\max}}{s})$ with $|\gD_P^{(2)}|= 30$ and $|\gD_Q^{(1)}|= 100$. This is likely because we do not have enough samples from $P$ to represent $\push{Q}{s}$ well enough via a weighted empirical distribution of the form $\push{\hat P_n^\vw}{s}$. More broadly, this underscores one of the limitations of our approach. The possibility of our method hurting coverage in some cases is not surprising, since we tackle the most general distribution shift case, with no prior information about the shift mechanics. In that setting, there is always a risk that optimizing our methods could negatively impact coverage. However, the results show a positive trend, indicating that these risks tend to diminish as the number of available samples increases.

\begin{table}
    \centering
    \caption{
        Total coverage gap on iWildCam \textcolor{black}{with ResNet-50} for varying number of calibration samples from $P$ ($|\gD_P^{(2)}|$) and $Q$ ($|\gD_Q^{(1)}|$) for our method with $(\push{Q^{\min}}{s}, \push{Q^{\max}}{s})$. We highlight in blue, the cases where total coverage improved by more than one standard deviation. We report mean and standard deviation across 10 random seeds. Lower is better.
    }
    \label{tab:iwildcam_total_min_max}
    \resizebox{\linewidth}{!}{
        \begin{tabular}{l| c | ccccc}
        \toprule
        \multicolumn{2}{c}{} &\multicolumn{5}{c}{Number of unlabeled samples from $Q$} \\\cmidrule{3-7}
        \# samples from $P$ & Uncorrected & $|\gD_Q^{(1)}|= 30$ & $|\gD_Q^{(1)}|= 100$ &  $|\gD_Q^{(1)}|= 300$ &  $|\gD_Q^{(1)}|= 1000$ & $|\gD_Q^{(1)}|= 3000 $ \\\midrule
        $|\gD_P^{(2)}| = 30$ & $0.119_{\pm0.046}$ & $0.117_{\pm0.048}$ & $0.122_{\pm0.044}$ & $0.115_{\pm0.035}$ & $0.118_{\pm0.039}$ & $0.119_{\pm0.045}$ \\
        $|\gD_P^{(2)}| = 100$ & $0.143_{\pm0.022}$ & $0.107_{\pm0.019}$ & $\color{Blue} 0.102_{\pm0.017}$ & $\color{Blue} 0.099_{\pm0.015}$ & $\color{Blue} 0.096_{\pm0.016}$ & $\color{Blue} 0.101_{\pm0.018}$ \\
        $|\gD_P^{(2)}| = 300$ & $0.145_{\pm0.014}$ & $\color{Blue} 0.091_{\pm0.017}$ & $\color{Blue} 0.085_{\pm0.009}$ & $\color{Blue} 0.084_{\pm0.010}$ & $\color{Blue} 0.083_{\pm0.010}$ & $\color{Blue} 0.085_{\pm0.010}$ \\
        $|\gD_P^{(2)}| = 1000$ & $0.139_{\pm0.010}$ & $\color{Blue} 0.085_{\pm0.015}$ & $\color{Blue} 0.078_{\pm0.006}$ & $\color{Blue} 0.072_{\pm0.005}$ & $\color{Blue} 0.072_{\pm0.004}$ & $\color{Blue} 0.073_{\pm0.006}$ \\
        $|\gD_P^{(2)}| = 3000$ & $0.139_{\pm0.005}$ & $\color{Blue} 0.082_{\pm0.015}$ & $\color{Blue} 0.075_{\pm0.008}$ & $\color{Blue} 0.070_{\pm0.007}$ & $\color{Blue} 0.069_{\pm0.005}$ & $\color{Blue} 0.070_{\pm0.006}$ \\
        \bottomrule
        \end{tabular}
    }
\end{table}

\begin{table}
    \centering
    \caption{
        Total coverage gap on iWildCam \textcolor{black}{with ResNet-50} for varying number of calibration samples from $P$ ($|\gD_P^{(2)}|$) and $Q$ ($|\gD_Q^{(1)}|$) for our method with $(\push{Q^{f}}{s}, \push{Q^{U}}{s})$. We report mean and standard deviation across 10 random seeds. Lower is better.
    }
    \label{tab:iwildcam_total_random_model}
    \resizebox{\linewidth}{!}{
        \begin{tabular}{l | c | ccccc}
        \toprule
        \multicolumn{2}{c}{} &\multicolumn{5}{c}{Number of unlabeled samples from $Q$} \\\cmidrule{3-7}
        \# samples from $P$ & Uncorrected & $|\gD_Q^{(1)}|= 30$ & $|\gD_Q^{(1)}|= 100$ &  $|\gD_Q^{(1)}|= 300$ &  $|\gD_Q^{(1)}|= 1000$ & $|\gD_Q^{(1)}|= 3000 $ \\\midrule
        $|\gD_P^{(2)}| = 30$ & $0.119_{\pm0.046}$ & $0.112_{\pm0.044}$ & $0.117_{\pm0.042}$ & $0.109_{\pm0.033}$ & $0.112_{\pm0.037}$ & $0.113_{\pm0.044}$ \\
        $|\gD_P^{(2)}| = 100$ & $0.143_{\pm0.022}$ & $\color{Blue} 0.097_{\pm0.021}$ & $\color{Blue} 0.094_{\pm0.018}$ & $\color{Blue} 0.090_{\pm0.014}$ & $\color{Blue} 0.086_{\pm0.014}$ & $\color{Blue} 0.091_{\pm0.018}$ \\
        $|\gD_P^{(2)}| = 300$ & $0.145_{\pm0.014}$ & $\color{Blue} 0.083_{\pm0.018}$ & $\color{Blue} 0.076_{\pm0.010}$ & $\color{Blue} 0.073_{\pm0.009}$ & $\color{Blue} 0.073_{\pm0.010}$ & $\color{Blue} 0.073_{\pm0.011}$ \\
        $|\gD_P^{(2)}| = 1000$ & $0.139_{\pm0.010}$ & $\color{Blue} 0.072_{\pm0.018}$ & $\color{Blue} 0.067_{\pm0.008}$ & $\color{Blue} 0.061_{\pm0.006}$ & $\color{Blue} 0.059_{\pm0.004}$ & $\color{Blue} 0.060_{\pm0.005}$ \\
        $|\gD_P^{(2)}| = 3000$ & $0.139_{\pm0.005}$ & $\color{Blue} 0.072_{\pm0.017}$ & $\color{Blue} 0.064_{\pm0.009}$ & $\color{Blue} 0.058_{\pm0.007}$ & $\color{Blue} 0.057_{\pm0.004}$ & $\color{Blue} 0.058_{\pm0.005}$ \\
        \bottomrule
        \end{tabular}
    }
\end{table}

\subsubsection{Choice of Bound}
Before applying our methods, two key decisions must be made. The first is whether to use the weighted CDF formulation in (\ref{eq:unlabeled-density-weighted}) or the 1-Wasserstein distance formulation in (\ref{eq:unlabeled-w1-bound}). The second involves selecting the appropriate pair of auxiliary distributions. To guide these choices, we evaluate the total coverage gap achieved after optimization under each configuration. The results for the image classification datasets are summarized in Table~\ref{tab:total_cov_gap}.

\begin{table*}[t] 
    \centering
    \small
    \caption{
        Total coverage gap on iWildcam and ImageNet-C with severity levels 1, 3 and 5 comparing optimization via the weighted-CDF (\ref{eq:unlabeled-density-weighted}) and the 1-Wasserstein (\ref{eq:unlabeled-w1-bound}) bounds with free-form parametrization. The classifier is given by a \textcolor{black}{ResNet-50} and we consider both pairs $(\push{Q^{\min}}{s}, \push{Q^{\max}}{s})$ and $(\push{Q^{f}}{s}, \push{Q^{U}}{s})$. For ImageNet-C we report the average across all 15 corruptions. Lower is better.
    }
    \label{tab:total_cov_gap}
    \resizebox{\linewidth}{!}{
        \begin{tabular}{ll c c c c}
            \toprule
            & & iWildCam & ImageNet-C Sev.~1 & ImageNet-C Sev.~3 & ImageNet-C Sev.~5 \\
            \midrule
            & Uncorrected & $0.132_{\pm0.016}$ & $0.141_{\pm0.048}$ & $0.267_{\pm0.081}$ & $0.388_{\pm0.076}$ \\
            \midrule
            \multirow{2}{*}{\makecell{weighted \\ CDF}} & $({\min}, {\max})$  & $0.084_{\pm0.010}$ & $0.098_{\pm0.033}$ & $0.171_{\pm0.060}$ & $0.281_{\pm0.076}$ \\
            & $(f, U)$   & $\bm {0.073}_{\pm0.009}$ & $0.059_{\pm0.023}$ & $\bm {0.102}_{\pm0.036}$ & $\bm  {0.173}_{\pm0.071}$ \\
            \midrule
            \multirow{2}{*}{{\makecell{1-Wasserstein}}} & $({\min}, {\max})$   & $0.125_{\pm0.005}$ & $0.069_{\pm0.023}$ & $0.196_{\pm0.062}$ & $0.337_{\pm0.093}$ \\
            & $(f, U)$   & $0.125_{\pm0.005}$ & $\bm {0.044}_{\pm0.020}$ & $0.139_{\pm0.066}$ & $0.310_{\pm0.108}$ \\
            \bottomrule
        \end{tabular}
    }
\end{table*}

\subsubsection*{Weighted CDF or 1-Wasserstein} 
It is clear from the theoretical results that the weighted CDF version of the bound is provably tighter than the 1-Wasserstein distance. Therefore, one should expect (\ref{eq:unlabeled-density-weighted}) to produce better results, and this seems to be the case for most image classification datasets, with the exception of ImageNet-C with severity level 1. In contrast, for regression tasks, we observed the opposite trend: the 1-Wasserstein distance outperformed the weighted CDF bound. As shown in Figure~\ref{fig:regression_comparison}, the weighted CDF bound led to overcoverage when paired with $(\push{\hat Q_m^{f}}{s}, \push{\hat Q_m^{U}}{s})$ and produced less consistent results when used with $(\push{\hat Q_m^{\min}}{s}, \push{\hat Q_m^{\max}}{s})$. We conjecture that this discrepancy arises from the increased complexity of optimizing the weighted CDF bound, which may account for the divergent empirical outcomes. Indeed, the task of finding the density of scores needed for the weighted CDF distance is more intricate and precision-sensitive than simply identifying its maximum. As a result, in certain cases, the weighted CDF bound may underperform relative to its 1-Wasserstein counterpart.

\subsubsection*{Choice of auxiliary distributions} 
The tightness and practical utility of our bounds are influenced by the choice of auxiliary distributions $(\push{Q^{\downarrow}}{s}, \push{Q^{\uparrow}}{s})$: the closer these are to the true distribution $\push{Q}{s}$, the tighter the resulting bounds. However, our bounds have proven effective for learning the weights of $\push{P^{\vw}}{s}$ even when using uninformed and widely applicable auxiliary pairs such as $(\push{Q^{\min}}{s}, \push{Q^{\max}}{s})$ and $(\push{Q^{f}}{s}, \push{Q^{U}}{s})$, which in general do not bound $\push{Q}{s}$ tightly. 

The performance of the two auxiliary distribution pairs was comparable in most cases, with $(\push{Q^{f}}{s}, \push{Q^{U}}{s})$ generally achieving better coverage, albeit with a tendency to overcover. As mentioned in the main paper, $(\push{Q^{f}}{s}, \push{Q^{U}}{s})$ is motivated by the observation that $\push{Q^{U}}{s}$ tends to produce nonconformity scores higher than those from the true distribution $\push{Q}{s}$—it corresponds to an uninformative model in which the correct label is independent of the model output—while $\push{Q^{f}}{s}$ tends to yield lower nonconformity scores, reflecting a perfect model where the true class is sampled according to the model-assigned probabilities. 

Beyond this theoretical motivation, there is also a practical reason for the better performance of $(\push{Q^{f}}{s}, \push{Q^{U}}{s})$, which relates to the specific form of the nonconformity scores used—namely, one minus the model-assigned probability for each class. Under this scoring scheme, the main difference between the two pairs arises from the contrast between $\push{Q^{f}}{s}$ and $\push{Q^{\min}}{s}$. Since most classes tend to receive relatively high nonconformity scores, $\push{Q^{U}}{s}$ and $\push{Q^{\max}}{s}$ are typically quite similar. On the other hand, unless the model is highly confident, $\push{Q^{f}}{s}$ and $\push{Q^{\min}}{s}$ can differ substantially. This explains why $\push{Q^{f}}{s}$ may be more effective in practice. In particular, if we expect the model to have low accuracy under $Q$, $\push{Q^{\min}}{s}$ becomes overly conservative and diverges significantly from the true $\push{Q}{s}$, a pattern clearly illustrated in Figure~\ref{fig:emp_cdfs}. Therefore, we can improve performance by biasing $\push{Q^\downarrow}{s}$ towards lower values, for instance, by sampling from the model, potentially with a high temperature. 

\subsubsection*{Computing the Bounds}
We compute the 1-Wasserstein version of the bound by estimating $\max_{s_c \in \R}p_{\push{P}{s}}(s_c)$ with a  Gaussian KDE and computing $W_1(\push{\hat P_n}{s}, \push{\hat Q_m}{s})$ analytically. However, for the weighted CDF version of the bound we have a couple of options. The first is to treat the bound as a expectation under $\push{P}{s}$, which can be approximated via the $n$ samples ${(X_i, Y_i)}_{i=1}^n$ we have from $P$
\begin{align} \label{eq:exp}
     \Delta_{P,Q} \leq \int p_{\push{P}{s}}(s_c) \left|F_{\push{\hat P_n}{s}}(s_c) - F_{\push{\hat Q_m}{s}}(s_c) \right| &= \E_{s_c \sim \push{P}{s}} \left[ \left|F_{\push{\hat P_n}{s}}(s_c) - F_{\push{\hat Q_m}{s}}(s_c) \right| \right] \\
    & \nonumber \approx \sum_{i=1}^n \frac{1}{n} \left|F_{\push{\hat P_n}{s}}(s(X_i, Y_i)) - F_{\push{\hat Q_m}{s}}(s(X_i, Y_i)) \right|.
\end{align}
This gives tight estimates and is computationally cheap but did not prove useful as an optimization objective for learning a weighting scheme. Alternatively, we could compute the upper bound by numerical integration, which works well for unidimensional data. Since our nonconformity scores are bounded in $[0,1]$, we use a grid of equally spaced $K$ points $s_k$ to get the following estimate
\begin{equation} \label{eq:grid}
    \Delta_{P,Q} \leq \int p_{\push{P}{s}}(s_c) \left|F_{\push{\hat P_n}{s}}(s_c) - F_{\push{\hat Q_m}{s}}(s_c) \right| 
    \\ \approx \sum_{k=1}^K \frac{p_{\push{P}{s}}(s_k)}{K} \left|F_{\push{\hat P_n}{s}}(s_k) - F_{\push{\hat Q_m}{s}}(s_k) \right|.
\end{equation}
In this case, we have to estimate the probability $\push{P}{s}(s_k)$ for each of the points in the grid. We do that with a Gaussian KDE, using reflection \citep{jones1993simple} to deal with the boundaries in $[0,1]$. This proved a better optimization objective, facilitating the learning of the weighted distribution $\push{\hat P_n^\vw}{s}$. Thus, when computing the weighted-CDF version of the bound we use the numerical integration method as in (\ref{eq:grid}) for both training and evaluation; see Table~\ref{tab:coverage_gaps} for an analysis of the tightness of the bound (\ref{eq:grid}) in ImageNet-C with severity 5. Note that in all cases, we can replace $\push{\hat P_n}{s}$ with $\push{\hat P_n^\vw}{s}$.

In terms of complexity, the bounds require computing either weighted CDF or 1-Wasserstein distances, which are tractable for unidimensional variables like nonconformity scores. In both cases, we need to compute the difference between the empirical CDFs, which has overall time complexity $\gO((m+n)\log(m+n))$. The complexity here is dominated by the sorting operation needed to compute the difference between the empirical CDFs, but fortunately this operation is applied only to the score values and not to the weights, and thus we need to compute it only once during optimization. Finally, with the exception of the bound in (\ref{eq:exp}), we also need to estimate the density $p_{\push{P}{s}}$. We do so via a Gaussian KDE defined on the $n$ samples from $\push{P}{s}$, which has cost $\gO(k \cdot n),$ where $k$ is the number of points the KDE is evaluated on, e.g.~the grid size in (\ref{eq:grid}). We set the KDE bandwidth using Scott’s rule \cite{scott2015multivariate} but scale it by a factor of 0.1 in classification tasks to improve resolution in the tails.

\begin{figure}
    \centering
    \includegraphics[width=0.9\linewidth]{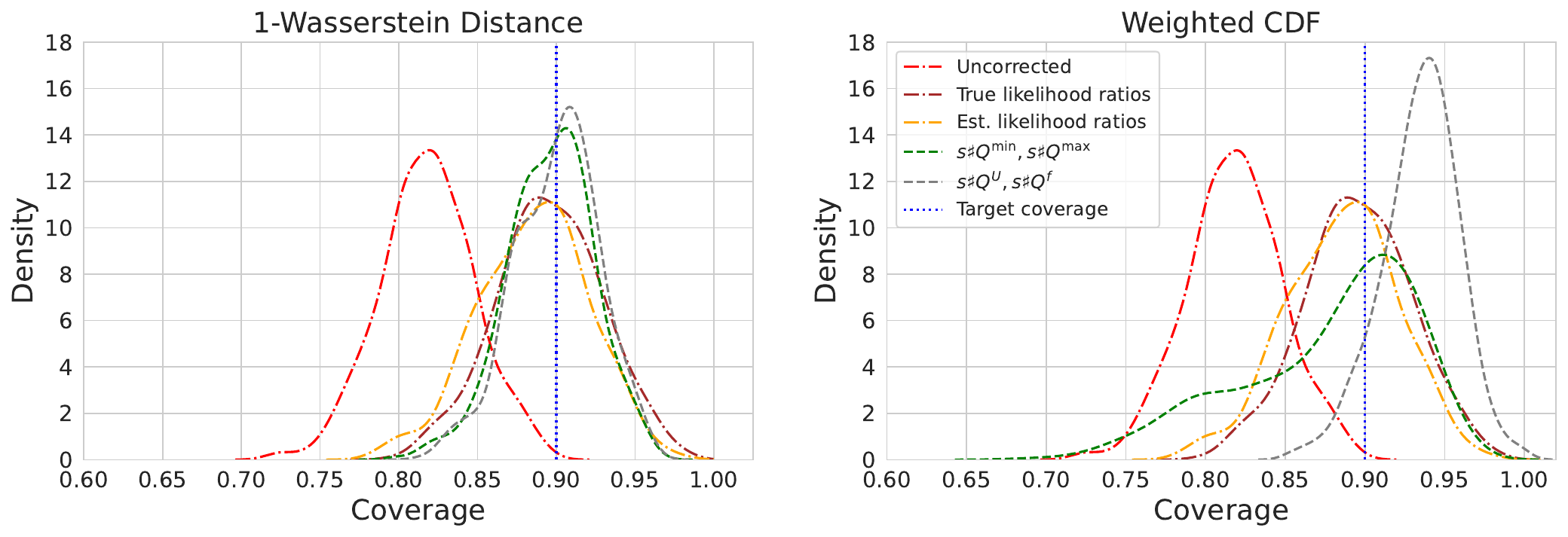}
    \caption{Distribution of coverage for the synthetic regression task across 500 simulations and target coverage rate of $90\%$ (blue vertical line). Results with the 1-Wasserstein distance formulation on the left and with the weighted CDF formulation on the right. The baselines remain the same in both plots. For ease of visualization, we plot the density estimated with a KDE fit to the 500 observations.}
    \label{fig:regression_comparison}
\end{figure} 

\subsubsection{Tightness of the Bounds}
Our upper bounds to the total coverage gap include terms that depend only on the auxiliary distributions, such as $\int_\R F_{\push{Q^{\downarrow}}{s}}(s_c) - F_{\push{Q^{\uparrow}}{s}}(s_c) ds_c$ in (\ref{eq:unlabeled-density-weighted}) and $\E_{\push{ Q^{\uparrow}}{s}}[S] - \E_{\push{Q^{\downarrow}}{s}}[S]$ in (\ref{eq:unlabeled-w1-bound}). Therefore, we cannot hope to have tight bounds, unless $\push{Q^{\uparrow}}{s}$ and $\push{Q^{\downarrow}}{s}$ are close to each other and sandwich $\push{Q}{s}$, i.e., satisfy the stochastic dominance relation $\push{Q^{\uparrow}}{s} \dominates \push{Q}{s} \dominates \push{Q^{\downarrow}}{s}$. To illustrate this point, we evaluate our upper bounds on ImageNet-C with severity level 5 in Table~\ref{tab:coverage_gaps}, where we can see the upper bounds constructed with the auxiliary distributions are relatively loose, as expected.

As discussed throughout the paper and demonstrated in the experiments, these bounds, although not tight, are still effective in mitigating the coverage gap by serving as a practical optimization objective for learning $\push{P^\vw}{s}$. Nevertheless, we conjecture that there is still room to improve the tightness of these bounds. We leave further improvements for future work, but one promising direction is to learn a transformation of the scores under $Q$ jointly with the weights of $\push{P^\vw}{s}$. Although this approach would still require auxiliary distributions to evaluate our bounds, it could yield tighter estimates, for instance, by reducing the gap between $\push{Q^{\min}}{s}$ and $\push{Q^{\max}}{s}$.

\begin{table}[ht]
    \centering
    \vspace{1cm}
    \caption{Upper bounds to the total coverage gap for ImageNet-C with severity level 5. For each pair of auxiliary distributions we consider, we have the total coverage gap $\Delta_{P^\vw, Q}$ after optimization, and the weighted-CDF upper bound computed with unlabeled samples and no DKW correction. We report mean and standard deviation across 10 random seeds. Lower is better. }
    \label{tab:coverage_gaps}
    \begin{tabular}{c | c c | c c c}
    Corruption & $\Delta_{P^\vw, Q}$ & $(\push{Q^{\min}}{s}, \push{Q^{\max}}{s})$ & $\Delta_{P^\vw, Q}$ & $(\push{Q^{f}}{s}, \push{Q^{U}}{s})$ \\
    \midrule
    Gauss & $0.377_{\pm0.026}$ & $0.445_{\pm0.014}$ & $0.264_{\pm0.033}$ & $0.291_{\pm0.019}$ \\
    Shot & $0.378_{\pm0.024}$ & $0.441_{\pm0.017}$ & $0.291_{\pm0.028}$ & $0.319_{\pm0.015}$ \\
    Impul & $0.378_{\pm0.026}$ & $0.450_{\pm0.013}$ & $0.274_{\pm0.027}$ & $0.306_{\pm0.012}$ \\
    Defoc & $0.258_{\pm0.030}$ & $0.328_{\pm0.022}$ & $0.094_{\pm0.013}$ & $0.113_{\pm0.018}$ \\
    Glass & $0.316_{\pm0.030}$ & $0.389_{\pm0.019}$ & $0.167_{\pm0.032}$ & $0.208_{\pm0.021}$ \\
    Motion & $0.291_{\pm0.028}$ & $0.354_{\pm0.019}$ & $0.168_{\pm0.023}$ & $0.209_{\pm0.013}$ \\
    Zoom & $0.252_{\pm0.027}$ & $0.306_{\pm0.024}$ & $0.153_{\pm0.027}$ & $0.190_{\pm0.030}$ \\
    Snow & $0.289_{\pm0.026}$ & $0.333_{\pm0.027}$ & $0.201_{\pm0.026}$ & $0.226_{\pm0.025}$ \\
    Frost & $0.254_{\pm0.026}$ & $0.298_{\pm0.025}$ & $0.162_{\pm0.023}$ & $0.189_{\pm0.021}$ \\
    Fog & $0.246_{\pm0.028}$ & $0.296_{\pm0.025}$ & $0.166_{\pm0.029}$ & $0.202_{\pm0.028}$ \\
    Bright & $0.095_{\pm0.009}$ & $0.098_{\pm0.011}$ & $0.056_{\pm0.009}$ & $0.060_{\pm0.009}$ \\
    Contr & $0.327_{\pm0.034}$ & $0.413_{\pm0.015}$ & $0.114_{\pm0.034}$ & $0.134_{\pm0.016}$ \\
    Elastic & $0.298_{\pm0.025}$ & $0.354_{\pm0.026}$ & $0.232_{\pm0.023}$ & $0.275_{\pm0.024}$ \\
    Pixel & $0.257_{\pm0.027}$ & $0.314_{\pm0.022}$ & $0.147_{\pm0.021}$ & $0.184_{\pm0.009}$ \\
    Jpeg & $0.196_{\pm0.023}$ & $0.230_{\pm0.030}$ & $0.106_{\pm0.013}$ & $0.119_{\pm0.017}$ \\
    \end{tabular}
\end{table}

\centering
\begin{sidewaystable}
    \begin{minipage}{1.1\paperheight}
    \caption{
    \small
        \textbf{Coverage on ImageNet-C with ResNet-50 with and without label shift (see Section~\ref{sec:imagenetc})}. Results for uncorrected distributions, calibrating and testing on samples \\ from $Q$ (Oracle), estimated likelihood ratios (LR), the methods of \citeauthor{kasa2024adapting} and \citeauthor{gibbs2023conformal}, and weights learned via our weighted CDF objective (\ref{eq:unlabeled-density-weighted}), including $({\min}, {\max})$ \\ and $(f, U)$ variants with free-form (FF) and weight function (WF) parametrizations. Target coverage of $90\%$. We report mean and standard  deviation across 10 random seeds.
    }
    \label{tab:imagenetc_cov_big}
    \resizebox{0.8125\paperheight}{!}{
        \begin{tabular}{llcccccccccccccccgg}
            \toprule
            & &\multicolumn{3}{c}{Noise} &\multicolumn{4}{c}{Blur} &\multicolumn{4}{c}{Weather} &\multicolumn{4}{c}{Digital} &\cellcolor[HTML]{FFFFFF} &\cellcolor[HTML]{FFFFFF} \\\cmidrule{3-17}
            & &Gauss &Shot &Impul &Defoc &Glass &Motion &Zoom &Snow &Frost &Fog &Bright &Contr &Elastic &Pixel &Jpeg &Avg. Cov. &Avg. Size \\\midrule
            \multirow{9}{*}{\rotatebox[origin=c]{90}{\makecell{Severity 1 \\ no label shift}}} & Uncorrected &$78.0_{\pm3.1}$&$76.8_{\pm3.2}$&$68.2_{\pm3.8}$&$78.4_{\pm3.5}$&$73.2_{\pm3.6}$&$82.2_{\pm2.8}$&$71.8_{\pm3.7}$&$73.2_{\pm3.3}$&$78.8_{\pm3.0}$&$80.3_{\pm3.2}$&$88.5_{\pm1.8}$&$82.6_{\pm2.9}$&$83.5_{\pm2.6}$&$81.7_{\pm2.7}$&$83.4_{\pm2.6}$&$78.7_{\pm6.0}$&$2.7_{\pm0.7}$ \\
            & Oracle &$89.7_{\pm1.7}$&$89.2_{\pm1.9}$&$89.7_{\pm1.6}$&$89.7_{\pm2.1}$&$89.8_{\pm1.5}$&$90.5_{\pm1.6}$&$90.1_{\pm1.4}$&$89.2_{\pm0.8}$&$89.7_{\pm2.1}$&$89.4_{\pm1.6}$&$90.1_{\pm1.2}$&$90.6_{\pm1.5}$&$90.4_{\pm1.3}$&$89.5_{\pm1.3}$&$89.7_{\pm1.7}$&$89.8_{\pm1.6}$ &$10.7_{\pm7.2}$ \\
            & LR &$82.2_{\pm3.2}$&$82.6_{\pm4.2}$&$77.3_{\pm5.7}$&$85.2_{\pm3.9}$&$81.2_{\pm4.3}$&$86.0_{\pm3.6}$&$81.3_{\pm4.3}$&$79.8_{\pm4.0}$&$84.0_{\pm4.2}$&$86.2_{\pm4.1}$&$89.5_{\pm2.2}$&$86.8_{\pm3.4}$&$87.3_{\pm3.5}$&$85.4_{\pm3.0}$&$86.5_{\pm3.2}$&$84.1_{\pm4.9}$&$5.2_{\pm3.8}$ \\
            & \citeauthor{kasa2024adapting} &$96.2_{\pm0.8}$&$96.1_{\pm0.8}$&$95.2_{\pm1.2}$&$97.5_{\pm0.7}$&$96.4_{\pm0.9}$&$97.3_{\pm0.7}$&$95.6_{\pm1.0}$&$95.3_{\pm1.0}$&$96.6_{\pm0.8}$&$97.1_{\pm0.7}$&$97.6_{\pm0.5}$&$97.4_{\pm0.6}$&$97.5_{\pm0.6}$&$96.6_{\pm0.7}$&$97.0_{\pm0.7}$&$96.6_{\pm1.1}$&$42.4_{\pm19.2}$ \\
            & \citeauthor{gibbs2023conformal} &$88.1_{\pm2.0}$&$87.8_{\pm1.8}$&$86.1_{\pm2.0}$&$89.8_{\pm1.8}$&$87.4_{\pm2.1}$&$90.6_{\pm2.3}$&$87.0_{\pm2.8}$&$85.4_{\pm3.0}$&$89.1_{\pm2.3}$&$89.1_{\pm1.9}$&$92.4_{\pm2.1}$&$89.7_{\pm2.6}$&$91.3_{\pm2.4}$&$89.0_{\pm1.7}$&$89.9_{\pm2.2}$&$88.9_{\pm2.8}$&$548.5_{\pm36.7}$ \\
            \cmidrule{2-19}
            & FF $({\min}, {\max})$ &$91.1_{\pm3.8}$&$89.5_{\pm5.1}$&$88.0_{\pm3.6}$&$92.9_{\pm3.0}$&$89.0_{\pm4.6}$&$92.9_{\pm2.9}$&$89.5_{\pm4.0}$&$88.9_{\pm3.4}$&$91.2_{\pm3.9}$&$92.8_{\pm3.7}$&$93.9_{\pm2.9}$&$93.5_{\pm2.3}$&$92.8_{\pm4.0}$&$91.7_{\pm3.4}$&$93.0_{\pm2.6}$&$91.4_{\pm3.9}$&$15.1_{\pm9.7}$ \\
            & WF $({\min}, {\max})$ &$93.7_{\pm4.4}$&$92.5_{\pm9.2}$&$92.6_{\pm3.4}$&$96.3_{\pm2.0}$&$94.5_{\pm2.4}$&$96.8_{\pm1.5}$&$93.6_{\pm2.9}$&$93.7_{\pm2.6}$&$95.7_{\pm1.9}$&$88.5_{\pm19.6}$&$95.4_{\pm5.6}$&$96.8_{\pm1.4}$&$96.9_{\pm1.5}$&$95.9_{\pm2.0}$&$96.5_{\pm1.5}$&$94.6_{\pm6.3}$&$34.7_{\pm20.6}$ \\
            & FF $(f, U)$ &$93.1_{\pm2.5}$&$91.1_{\pm3.0}$&$88.6_{\pm3.8}$&$94.2_{\pm2.2}$&$91.4_{\pm1.7}$&$94.5_{\pm2.6}$&$91.7_{\pm3.1}$&$90.6_{\pm3.2}$&$93.4_{\pm2.7}$&$94.3_{\pm2.1}$&$95.6_{\pm1.4}$&$94.4_{\pm1.6}$&$95.0_{\pm1.9}$&$93.6_{\pm2.0}$&$93.9_{\pm2.2}$&$93.0_{\pm3.0}$&$19.2_{\pm11.4}$ \\
            & WF $(f, U)$ &$95.3_{\pm2.0}$&$95.1_{\pm2.0}$&$93.2_{\pm2.6}$&$96.4_{\pm1.9}$&$94.7_{\pm2.5}$&$96.7_{\pm1.6}$&$94.3_{\pm2.2}$&$94.2_{\pm2.1}$&$95.8_{\pm1.9}$&$96.0_{\pm2.0}$&$97.5_{\pm0.9}$&$96.8_{\pm1.5}$&$96.9_{\pm1.5}$&$96.0_{\pm1.6}$&$96.6_{\pm1.6}$&$95.7_{\pm2.2}$&$36.6_{\pm20.1}$ \\
            \midrule
            \multirow{9}{*}{\rotatebox[origin=c]{90}{\makecell{Severity 1 \\ with label shift}}} & Uncorrected &$78.8_{\pm5.3}$&$77.4_{\pm6.2}$&$69.2_{\pm8.0}$&$78.9_{\pm7.3}$&$76.7_{\pm6.0}$&$83.8_{\pm5.6}$&$72.1_{\pm7.8}$&$72.5_{\pm9.5}$&$79.2_{\pm7.3}$&$78.7_{\pm4.5}$&$89.5_{\pm2.4}$&$82.9_{\pm8.0}$&$84.6_{\pm6.0}$&$80.0_{\pm7.2}$&$83.3_{\pm5.5}$&$79.2_{\pm8.2}$&$2.7_{\pm0.8}$ \\
            & Oracle &$89.0_{\pm3.7}$&$87.1_{\pm5.5}$&$91.7_{\pm4.7}$&$93.7_{\pm1.1}$&$90.5_{\pm4.2}$&$90.8_{\pm4.7}$&$90.4_{\pm3.3}$&$90.9_{\pm6.7}$&$89.0_{\pm3.3}$&$87.2_{\pm6.3}$&$89.9_{\pm2.7}$&$93.2_{\pm3.2}$&$90.1_{\pm5.5}$&$90.1_{\pm3.5}$&$90.4_{\pm3.8}$&$90.3_{\pm4.5}$&$15.0_{\pm19.9}$ \\
            & LR &$85.2_{\pm5.5}$&$84.8_{\pm4.4}$&$78.2_{\pm6.9}$&$82.9_{\pm8.7}$&$80.1_{\pm6.4}$&$88.0_{\pm6.6}$&$82.1_{\pm5.1}$&$79.2_{\pm9.8}$&$85.3_{\pm7.7}$&$85.8_{\pm5.8}$&$89.0_{\pm7.2}$&$89.3_{\pm3.7}$&$85.0_{\pm6.4}$&$84.6_{\pm7.5}$&$87.4_{\pm5.0}$&$84.5_{\pm7.1}$&$6.1_{\pm5.5}$ \\
            & \citeauthor{kasa2024adapting} &$78.8_{\pm5.3}$&$77.4_{\pm6.2}$&$69.3_{\pm7.9}$&$78.9_{\pm7.3}$&$77.0_{\pm5.9}$&$83.8_{\pm5.6}$&$72.1_{\pm7.8}$&$72.5_{\pm9.5}$&$79.2_{\pm7.3}$&$78.7_{\pm4.5}$&$89.5_{\pm2.4}$&$82.9_{\pm8.0}$&$84.6_{\pm6.0}$&$80.0_{\pm7.2}$&$83.3_{\pm5.5}$&$79.2_{\pm8.1}$&$2.7_{\pm0.8}$ \\
            & \citeauthor{gibbs2023conformal} &$87.8_{\pm2.6}$&$87.8_{\pm4.3}$&$84.6_{\pm2.2}$&$89.0_{\pm2.1}$&$87.9_{\pm3.2}$&$88.9_{\pm2.4}$&$85.7_{\pm4.2}$&$86.7_{\pm5.7}$&$89.4_{\pm4.3}$&$89.3_{\pm4.7}$&$91.9_{\pm2.9}$&$90.4_{\pm2.2}$&$89.2_{\pm3.4}$&$89.3_{\pm1.9}$&$90.9_{\pm4.3}$&$88.6_{\pm3.8}$&$552.3_{\pm46.0}$ \\
            \cmidrule{2-19}
            & FF $({\min}, {\max})$ &$90.1_{\pm3.6}$&$89.5_{\pm4.1}$&$85.7_{\pm9.5}$&$93.7_{\pm3.8}$&$87.9_{\pm7.2}$&$93.9_{\pm4.1}$&$87.8_{\pm5.8}$&$88.1_{\pm6.2}$&$91.0_{\pm5.3}$&$91.1_{\pm4.7}$&$94.1_{\pm3.3}$&$94.8_{\pm3.8}$&$94.5_{\pm3.3}$&$90.5_{\pm5.1}$&$91.8_{\pm5.8}$&$91.0_{\pm5.7}$&$14.3_{\pm10.2}$ \\
            & WF $({\min}, {\max})$ &$93.9_{\pm2.9}$&$94.8_{\pm2.6}$&$93.1_{\pm2.4}$&$95.1_{\pm3.2}$&$94.7_{\pm3.5}$&$96.7_{\pm1.6}$&$92.0_{\pm5.2}$&$92.8_{\pm4.6}$&$94.9_{\pm3.2}$&$92.4_{\pm11.5}$&$85.1_{\pm24.4}$&$90.2_{\pm21.0}$&$96.7_{\pm2.7}$&$96.0_{\pm2.1}$&$96.6_{\pm1.8}$&$93.7_{\pm9.3}$&$33.1_{\pm21.5}$ \\
            & FF $(f, U)$ &$92.3_{\pm3.2}$&$92.5_{\pm3.5}$&$89.9_{\pm4.8}$&$95.6_{\pm2.6}$&$92.9_{\pm3.6}$&$94.4_{\pm3.6}$&$91.1_{\pm5.2}$&$93.2_{\pm2.8}$&$92.2_{\pm2.5}$&$94.5_{\pm2.4}$&$95.3_{\pm3.1}$&$95.1_{\pm3.0}$&$96.1_{\pm2.0}$&$92.8_{\pm4.4}$&$94.4_{\pm2.4}$&$93.5_{\pm3.6}$&$20.7_{\pm14.6}$ \\
            & WF $(f, U)$ &$93.9_{\pm3.1}$&$95.8_{\pm2.5}$&$93.5_{\pm3.5}$&$96.9_{\pm3.9}$&$95.4_{\pm2.6}$&$97.0_{\pm2.4}$&$94.2_{\pm3.9}$&$93.6_{\pm2.9}$&$95.7_{\pm3.7}$&$95.8_{\pm3.0}$&$89.4_{\pm25.3}$&$96.3_{\pm4.0}$&$97.4_{\pm1.8}$&$95.5_{\pm2.7}$&$96.4_{\pm2.4}$&$95.1_{\pm7.1}$&$36.3_{\pm21.8}$ \\
            \midrule
            \multirow{9}{*}{\rotatebox[origin=c]{90}{\makecell{Severity 3 \\ no label shift}}} & Uncorrected &$50.3_{\pm4.6}$&$46.7_{\pm4.4}$&$46.9_{\pm4.7}$&$57.3_{\pm5.4}$&$30.3_{\pm4.2}$&$56.6_{\pm4.7}$&$54.2_{\pm4.6}$&$53.0_{\pm4.1}$&$48.9_{\pm4.3}$&$67.0_{\pm4.2}$&$85.7_{\pm2.3}$&$66.3_{\pm4.3}$&$73.4_{\pm3.2}$&$66.1_{\pm3.9}$&$78.3_{\pm3.1}$&$58.7_{\pm14.3}$&$3.3_{\pm1.2}$ \\
            & Oracle &$89.5_{\pm1.4}$&$89.5_{\pm2.4}$&$90.0_{\pm2.0}$&$90.0_{\pm1.8}$&$89.5_{\pm2.3}$&$90.7_{\pm1.7}$&$90.4_{\pm1.6}$&$89.3_{\pm0.9}$&$89.5_{\pm2.1}$&$89.6_{\pm1.3}$&$90.0_{\pm1.6}$&$89.8_{\pm1.5}$&$90.4_{\pm1.7}$&$89.7_{\pm1.3}$&$89.9_{\pm1.4}$&$89.9_{\pm1.7}$&$80.1_{\pm71.4}$ \\
            & LR &$64.2_{\pm7.4}$&$62.8_{\pm8.3}$&$63.7_{\pm9.1}$&$75.3_{\pm8.5}$&$48.2_{\pm10.3}$&$71.5_{\pm8.6}$&$69.8_{\pm7.6}$&$65.7_{\pm7.0}$&$63.8_{\pm8.1}$&$78.2_{\pm5.9}$&$87.4_{\pm2.7}$&$78.3_{\pm5.6}$&$80.8_{\pm4.7}$&$76.2_{\pm5.3}$&$82.9_{\pm4.2}$&$71.3_{\pm11.9}$&$12.3_{\pm13.1}$ \\
            & \citeauthor{kasa2024adapting} &$92.4_{\pm1.9}$&$91.6_{\pm2.1}$&$92.0_{\pm2.0}$&$97.2_{\pm0.9}$&$87.6_{\pm3.2}$&$94.4_{\pm1.5}$&$92.9_{\pm1.7}$&$90.2_{\pm2.1}$&$90.8_{\pm2.2}$&$95.4_{\pm1.2}$&$97.4_{\pm0.5}$&$96.3_{\pm1.1}$&$95.7_{\pm1.0}$&$94.5_{\pm1.3}$&$96.5_{\pm0.9}$&$93.7_{\pm3.2}$&$119.3_{\pm65.4}$ \\
            & \citeauthor{gibbs2023conformal} &$83.1_{\pm1.8}$&$81.7_{\pm2.3}$&$83.2_{\pm2.7}$&$89.3_{\pm1.8}$&$83.5_{\pm2.0}$&$87.3_{\pm2.5}$&$83.6_{\pm2.5}$&$82.2_{\pm3.0}$&$83.0_{\pm2.1}$&$85.7_{\pm2.1}$&$91.0_{\pm1.6}$&$86.3_{\pm2.7}$&$89.1_{\pm2.2}$&$85.5_{\pm2.1}$&$88.8_{\pm2.2}$&$85.6_{\pm3.6}$&$635.7_{\pm66.3}$ \\
            \cmidrule{2-19}
            & FF $({\min}, {\max})$ &$85.1_{\pm5.7}$&$83.3_{\pm6.7}$&$84.2_{\pm6.3}$&$92.5_{\pm4.3}$&$73.5_{\pm9.9}$&$89.3_{\pm4.5}$&$86.6_{\pm4.7}$&$83.6_{\pm4.8}$&$83.1_{\pm6.2}$&$92.3_{\pm3.2}$&$96.8_{\pm1.2}$&$92.9_{\pm3.2}$&$93.4_{\pm2.6}$&$91.2_{\pm3.6}$&$95.1_{\pm2.0}$&$88.2_{\pm7.7}$&$63.9_{\pm41.3}$ \\
            & WF $({\min}, {\max})$ &$85.7_{\pm6.4}$&$83.6_{\pm7.1}$&$84.6_{\pm6.8}$&$92.7_{\pm4.7}$&$74.6_{\pm9.5}$&$89.8_{\pm5.0}$&$87.3_{\pm5.5}$&$84.4_{\pm5.9}$&$83.4_{\pm6.6}$&$93.2_{\pm2.9}$&$97.2_{\pm1.2}$&$84.6_{\pm28.8}$&$94.1_{\pm2.6}$&$82.9_{\pm27.9}$&$86.7_{\pm27.8}$&$87.0_{\pm14.1}$&$69.7_{\pm46.6}$ \\
            & FF $(f, U)$ &$87.9_{\pm4.3}$&$86.1_{\pm4.9}$&$86.9_{\pm4.6}$&$95.0_{\pm2.6}$&$80.0_{\pm6.3}$&$91.3_{\pm3.3}$&$89.0_{\pm3.8}$&$86.0_{\pm4.1}$&$85.6_{\pm4.8}$&$93.4_{\pm2.7}$&$97.0_{\pm1.2}$&$94.1_{\pm2.7}$&$94.2_{\pm2.5}$&$92.1_{\pm3.1}$&$95.6_{\pm1.9}$&$90.3_{\pm5.9}$&$79.8_{\pm48.9}$ \\
            & WF $(f, U)$ &$87.4_{\pm5.1}$&$85.3_{\pm5.5}$&$86.4_{\pm5.4}$&$94.7_{\pm2.7}$&$79.7_{\pm6.4}$&$91.2_{\pm3.8}$&$88.8_{\pm4.4}$&$85.8_{\pm4.7}$&$85.3_{\pm5.4}$&$93.4_{\pm2.7}$&$97.1_{\pm1.4}$&$94.1_{\pm2.7}$&$94.5_{\pm2.2}$&$92.3_{\pm2.9}$&$95.6_{\pm2.0}$&$90.1_{\pm6.2}$&$79.7_{\pm50.1}$ \\
            \midrule
            \multirow{9}{*}{\rotatebox[origin=c]{90}{\makecell{Severity 3 \\ with label shift}}} & Uncorrected &$51.4_{\pm8.6}$&$48.3_{\pm11.1}$&$47.5_{\pm8.6}$&$57.0_{\pm8.0}$&$31.4_{\pm8.2}$&$58.6_{\pm7.3}$&$53.5_{\pm7.7}$&$54.1_{\pm6.9}$&$49.8_{\pm9.9}$&$67.3_{\pm5.8}$&$86.6_{\pm4.2}$&$65.4_{\pm10.8}$&$75.4_{\pm8.0}$&$64.8_{\pm11.6}$&$76.5_{\pm7.1}$&$59.2_{\pm15.7}$&$3.2_{\pm1.2}$ \\
            & Oracle &$89.3_{\pm4.0}$&$89.5_{\pm3.9}$&$92.9_{\pm2.6}$&$91.8_{\pm3.0}$&$89.1_{\pm3.5}$&$89.9_{\pm4.0}$&$89.1_{\pm4.3}$&$90.1_{\pm3.4}$&$89.9_{\pm2.9}$&$88.0_{\pm5.2}$&$88.8_{\pm3.3}$&$93.0_{\pm4.2}$&$91.2_{\pm4.3}$&$90.7_{\pm2.6}$&$90.4_{\pm4.2}$&$90.2_{\pm3.8}$&$88.1_{\pm87.4}$ \\
            & LR &$69.2_{\pm8.8}$&$64.3_{\pm12.0}$&$64.6_{\pm11.4}$&$78.0_{\pm8.6}$&$49.9_{\pm12.9}$&$70.4_{\pm9.5}$&$72.9_{\pm5.2}$&$67.6_{\pm12.5}$&$63.3_{\pm13.4}$&$78.5_{\pm6.5}$&$88.7_{\pm6.6}$&$81.2_{\pm4.8}$&$77.2_{\pm6.6}$&$77.3_{\pm5.3}$&$84.1_{\pm6.0}$&$72.5_{\pm12.9}$&$14.3_{\pm15.8}$ \\
            & \citeauthor{kasa2024adapting} &$51.9_{\pm8.6}$&$48.3_{\pm11.1}$&$48.9_{\pm8.0}$&$57.6_{\pm8.5}$&$34.5_{\pm7.0}$&$58.9_{\pm7.6}$&$54.1_{\pm8.0}$&$54.3_{\pm6.9}$&$50.0_{\pm10.0}$&$67.3_{\pm5.8}$&$86.6_{\pm4.2}$&$66.4_{\pm10.6}$&$75.4_{\pm8.0}$&$65.1_{\pm11.3}$&$76.5_{\pm7.1}$&$59.7_{\pm15.2}$&$3.4_{\pm1.4}$ \\
            & \citeauthor{gibbs2023conformal} &$82.7_{\pm2.8}$&$84.6_{\pm2.8}$&$80.7_{\pm3.0}$&$89.4_{\pm2.6}$&$82.9_{\pm3.4}$&$86.5_{\pm3.0}$&$82.9_{\pm3.9}$&$81.5_{\pm4.8}$&$84.4_{\pm3.4}$&$86.0_{\pm4.6}$&$91.0_{\pm2.7}$&$87.8_{\pm4.7}$&$88.3_{\pm3.2}$&$86.0_{\pm4.6}$&$87.9_{\pm4.0}$&$85.5_{\pm4.6}$&$638.6_{\pm70.2}$ \\
            \cmidrule{2-19}
            & FF $({\min}, {\max})$ &$85.6_{\pm4.7}$&$84.9_{\pm8.5}$&$83.5_{\pm9.2}$&$92.9_{\pm4.1}$&$74.9_{\pm13.3}$&$88.9_{\pm5.9}$&$86.1_{\pm7.2}$&$83.8_{\pm6.1}$&$82.9_{\pm8.3}$&$91.7_{\pm4.6}$&$97.2_{\pm2.1}$&$93.8_{\pm4.0}$&$94.2_{\pm3.4}$&$90.1_{\pm4.7}$&$95.3_{\pm3.8}$&$88.4_{\pm8.6}$&$62.6_{\pm41.4}$ \\
            & WF $({\min}, {\max})$ &$84.8_{\pm6.8}$&$85.1_{\pm8.1}$&$86.2_{\pm6.7}$&$92.7_{\pm4.7}$&$74.6_{\pm11.2}$&$90.0_{\pm6.8}$&$85.9_{\pm12.0}$&$81.5_{\pm7.7}$&$82.4_{\pm6.0}$&$92.8_{\pm5.4}$&$88.1_{\pm27.5}$&$93.0_{\pm5.6}$&$94.3_{\pm3.3}$&$82.4_{\pm27.9}$&$95.5_{\pm2.6}$&$87.3_{\pm12.8}$&$69.4_{\pm47.3}$ \\
            & FF $(f, U)$ &$87.5_{\pm3.3}$&$85.6_{\pm5.9}$&$87.7_{\pm7.4}$&$95.4_{\pm3.4}$&$77.3_{\pm8.1}$&$91.7_{\pm4.2}$&$88.0_{\pm5.1}$&$88.3_{\pm4.1}$&$84.7_{\pm4.6}$&$93.3_{\pm4.4}$&$97.7_{\pm1.8}$&$93.8_{\pm3.9}$&$95.6_{\pm1.7}$&$92.4_{\pm3.8}$&$95.6_{\pm1.8}$&$90.3_{\pm6.8}$&$78.9_{\pm51.7}$ \\
            & WF $(f, U)$ &$86.4_{\pm5.6}$&$87.7_{\pm5.5}$&$87.8_{\pm6.5}$&$93.8_{\pm8.9}$&$81.1_{\pm9.9}$&$92.4_{\pm3.9}$&$88.8_{\pm5.7}$&$86.9_{\pm6.1}$&$86.5_{\pm8.0}$&$93.6_{\pm3.4}$&$96.9_{\pm3.0}$&$92.6_{\pm7.2}$&$94.7_{\pm3.3}$&$92.0_{\pm2.7}$&$95.6_{\pm2.2}$&$90.4_{\pm7.0}$&$80.0_{\pm51.4}$ \\
            \midrule
            \multirow{9}{*}{\rotatebox[origin=c]{90}{\makecell{Severity 5 \\ no label shift}}} & Uncorrected &$6.9_{\pm1.8}$&$8.3_{\pm1.9}$&$6.5_{\pm1.8}$&$30.2_{\pm5.4}$&$18.4_{\pm3.7}$&$26.8_{\pm4.3}$&$38.0_{\pm4.5}$&$29.4_{\pm3.7}$&$37.6_{\pm4.1}$&$40.5_{\pm4.3}$&$78.0_{\pm3.2}$&$9.6_{\pm2.6}$&$30.3_{\pm3.8}$&$36.2_{\pm4.4}$&$50.9_{\pm4.7}$&$29.8_{\pm18.9}$&$3.3_{\pm1.5}$ \\
            & Oracle &$89.5_{\pm2.3}$&$89.2_{\pm2.1}$&$89.5_{\pm2.6}$&$89.9_{\pm2.0}$&$89.6_{\pm2.5}$&$90.4_{\pm1.0}$&$89.9_{\pm1.7}$&$89.3_{\pm1.6}$&$89.2_{\pm1.5}$&$90.3_{\pm2.4}$&$89.7_{\pm1.5}$&$89.0_{\pm1.4}$&$90.4_{\pm2.4}$&$89.9_{\pm1.6}$&$90.3_{\pm1.5}$&$89.7_{\pm1.9}$&$338.6_{\pm190.9}$ \\
            & LR &$21.5_{\pm11.4}$&$23.1_{\pm13.0}$&$20.9_{\pm12.6}$&$57.7_{\pm15.6}$&$40.1_{\pm13.6}$&$48.4_{\pm13.9}$&$57.3_{\pm10.2}$&$47.1_{\pm10.5}$&$54.3_{\pm9.7}$&$57.5_{\pm9.1}$&$82.9_{\pm4.3}$&$32.8_{\pm16.5}$&$45.9_{\pm8.1}$&$55.6_{\pm10.7}$&$66.5_{\pm7.7}$&$47.4_{\pm20.2}$&$27.7_{\pm37.0}$ \\
            & \citeauthor{kasa2024adapting} &$76.1_{\pm6.1}$&$71.1_{\pm6.1}$&$75.9_{\pm6.2}$&$96.1_{\pm1.6}$&$87.4_{\pm3.8}$&$87.2_{\pm3.0}$&$89.0_{\pm2.7}$&$82.6_{\pm3.7}$&$86.5_{\pm2.9}$&$88.2_{\pm2.8}$&$97.0_{\pm0.7}$&$93.4_{\pm3.3}$&$78.3_{\pm3.6}$&$89.2_{\pm2.8}$&$93.2_{\pm1.9}$&$86.1_{\pm8.4}$&$282.0_{\pm159.1}$ \\
            & \citeauthor{gibbs2023conformal} &$86.9_{\pm2.2}$&$83.4_{\pm2.9}$&$86.3_{\pm1.9}$&$92.8_{\pm1.7}$&$88.6_{\pm1.7}$&$85.8_{\pm1.9}$&$82.2_{\pm2.8}$&$78.6_{\pm2.4}$&$81.1_{\pm1.9}$&$79.3_{\pm2.6}$&$87.5_{\pm2.4}$&$95.2_{\pm1.1}$&$75.1_{\pm3.4}$&$85.5_{\pm3.5}$&$85.4_{\pm1.7}$&$84.9_{\pm5.6}$&$754.4_{\pm102.7}$ \\
            \cmidrule{2-19}
            & FF $({\min}, {\max})$ &$48.7_{\pm15.2}$&$46.9_{\pm13.5}$&$48.2_{\pm15.1}$&$84.3_{\pm9.7}$&$68.5_{\pm12.8}$&$73.9_{\pm10.1}$&$78.8_{\pm7.8}$&$70.1_{\pm9.6}$&$76.2_{\pm7.9}$&$79.1_{\pm7.9}$&$95.6_{\pm1.9}$&$69.0_{\pm13.6}$&$67.2_{\pm8.4}$&$76.8_{\pm8.6}$&$86.1_{\pm6.0}$&$71.3_{\pm17.0}$&$128.5_{\pm93.0}$ \\
            & WF $({\min}, {\max})$ &$49.8_{\pm14.8}$&$46.9_{\pm14.1}$&$48.3_{\pm15.8}$&$86.1_{\pm7.5}$&$69.2_{\pm12.3}$&$73.6_{\pm10.5}$&$78.8_{\pm8.2}$&$70.0_{\pm10.1}$&$76.9_{\pm7.7}$&$78.9_{\pm8.3}$&$95.7_{\pm2.1}$&$67.9_{\pm14.8}$&$67.1_{\pm8.8}$&$76.8_{\pm9.0}$&$85.9_{\pm6.3}$&$71.5_{\pm17.1}$&$130.5_{\pm94.1}$ \\
            & FF $(f, U)$ &$66.2_{\pm10.8}$&$61.7_{\pm11.2}$&$65.8_{\pm10.8}$&$92.9_{\pm4.5}$&$81.2_{\pm8.1}$&$79.9_{\pm5.6}$&$81.8_{\pm5.7}$&$74.9_{\pm6.6}$&$80.3_{\pm5.4}$&$81.5_{\pm6.1}$&$96.0_{\pm1.9}$&$88.5_{\pm10.9}$&$69.8_{\pm6.7}$&$83.0_{\pm4.9}$&$88.2_{\pm4.3}$&$79.5_{\pm12.1}$&$205.7_{\pm164.4}$ \\
            & WF $(f, U)$ &$66.4_{\pm10.8}$&$57.2_{\pm11.3}$&$63.1_{\pm11.8}$&$93.3_{\pm3.9}$&$79.9_{\pm8.2}$&$79.6_{\pm7.1}$&$82.3_{\pm5.4}$&$73.3_{\pm7.9}$&$79.2_{\pm6.1}$&$81.1_{\pm6.8}$&$96.2_{\pm1.6}$&$90.6_{\pm9.7}$&$69.5_{\pm7.5}$&$80.7_{\pm6.5}$&$87.7_{\pm5.0}$&$78.7_{\pm13.1}$&$200.4_{\pm164.2}$ \\
            \midrule
            \multirow{9}{*}{\rotatebox[origin=c]{90}{\makecell{Severity 5 \\ with label shift}}} & Uncorrected  &$5.0_{\pm2.5}$&$8.6_{\pm6.0}$&$5.7_{\pm2.3}$&$29.2_{\pm6.5}$&$19.8_{\pm6.5}$&$26.9_{\pm7.4}$&$36.5_{\pm8.6}$&$33.1_{\pm7.8}$&$37.9_{\pm11.5}$&$42.9_{\pm7.8}$&$79.4_{\pm5.4}$&$10.4_{\pm6.3}$&$29.0_{\pm9.6}$&$35.5_{\pm11.1}$&$48.1_{\pm8.4}$&$29.8_{\pm20.1}$&$3.3_{\pm1.5}$ \\
            & Oracle &$87.4_{\pm4.0}$&$93.0_{\pm2.8}$&$91.0_{\pm3.5}$&$92.6_{\pm2.9}$&$88.8_{\pm6.2}$&$91.3_{\pm3.4}$&$89.3_{\pm4.2}$&$90.6_{\pm2.7}$&$89.3_{\pm4.7}$&$88.6_{\pm3.8}$&$91.2_{\pm2.8}$&$90.7_{\pm3.5}$&$90.7_{\pm5.3}$&$92.1_{\pm2.8}$&$91.0_{\pm3.5}$&$90.5_{\pm4.0}$&$349.0_{\pm220.0}$ \\
            & LR &$18.1_{\pm9.1}$&$23.1_{\pm11.8}$&$21.8_{\pm14.8}$&$58.9_{\pm14.6}$&$39.1_{\pm16.2}$&$46.2_{\pm15.7}$&$61.7_{\pm6.9}$&$49.2_{\pm14.3}$&$55.0_{\pm14.5}$&$58.8_{\pm8.9}$&$83.0_{\pm8.5}$&$33.3_{\pm16.8}$&$41.7_{\pm10.0}$&$53.6_{\pm10.1}$&$66.3_{\pm9.4}$&$47.3_{\pm21.2}$&$27.7_{\pm33.5}$ \\
            & \citeauthor{kasa2024adapting} &$6.2_{\pm4.0}$&$9.0_{\pm6.1}$&$7.4_{\pm3.2}$&$32.3_{\pm8.4}$&$23.3_{\pm7.4}$&$28.4_{\pm9.1}$&$38.6_{\pm9.0}$&$33.6_{\pm7.7}$&$38.3_{\pm11.6}$&$43.2_{\pm7.6}$&$79.4_{\pm5.4}$&$13.4_{\pm6.8}$&$29.1_{\pm9.4}$&$37.4_{\pm10.7}$&$49.1_{\pm8.4}$&$31.2_{\pm19.9}$&$4.1_{\pm2.3}$ \\
            & \citeauthor{gibbs2023conformal} &$87.0_{\pm2.8}$&$83.8_{\pm3.8}$&$85.8_{\pm3.9}$&$91.8_{\pm1.7}$&$88.1_{\pm2.4}$&$85.5_{\pm3.5}$&$81.0_{\pm3.9}$&$77.3_{\pm5.0}$&$82.0_{\pm4.1}$&$81.2_{\pm4.9}$&$88.3_{\pm3.4}$&$93.8_{\pm2.4}$&$77.0_{\pm3.3}$&$84.2_{\pm6.1}$&$85.0_{\pm3.2}$&$84.8_{\pm5.8}$&$753.5_{\pm100.4}$\\
            \cmidrule{2-19}
            & FF $({\min}, {\max})$ &$47.6_{\pm18.1}$&$50.1_{\pm19.2}$&$48.0_{\pm23.8}$&$86.3_{\pm8.2}$&$71.4_{\pm15.6}$&$74.9_{\pm9.8}$&$78.1_{\pm9.0}$&$73.0_{\pm11.2}$&$76.4_{\pm12.3}$&$81.0_{\pm8.2}$&$95.5_{\pm2.7}$&$68.4_{\pm16.7}$&$64.7_{\pm13.9}$&$79.2_{\pm9.1}$&$84.7_{\pm5.8}$&$72.0_{\pm18.8}$&$127.9_{\pm94.4}$ \\
            & WF $({\min}, {\max})$ &$48.8_{\pm16.0}$&$47.0_{\pm18.0}$&$51.5_{\pm16.9}$&$81.9_{\pm14.7}$&$70.5_{\pm14.4}$&$72.8_{\pm13.4}$&$81.4_{\pm10.6}$&$64.1_{\pm13.2}$&$76.1_{\pm8.7}$&$81.2_{\pm10.7}$&$87.3_{\pm27.2}$&$66.4_{\pm16.6}$&$63.6_{\pm14.2}$&$77.4_{\pm11.5}$&$85.2_{\pm9.4}$&$70.4_{\pm19.2}$&$129.2_{\pm96.3}$ \\
            & FF $(f, U)$ &$61.5_{\pm16.6}$&$60.3_{\pm13.1}$&$66.1_{\pm18.5}$&$92.9_{\pm5.0}$&$77.2_{\pm10.7}$&$80.3_{\pm6.5}$&$81.8_{\pm8.4}$&$78.4_{\pm3.4}$&$80.1_{\pm4.9}$&$82.7_{\pm6.9}$&$96.7_{\pm1.6}$&$90.3_{\pm9.6}$&$75.4_{\pm10.4}$&$84.3_{\pm5.1}$&$89.3_{\pm4.7}$&$79.8_{\pm13.9}$&$206.0_{\pm168.3}$ \\
            & WF $(f, U)$ &$62.1_{\pm14.3}$&$55.6_{\pm16.4}$&$66.0_{\pm6.9}$&$90.7_{\pm15.5}$&$83.3_{\pm11.3}$&$81.1_{\pm8.8}$&$80.4_{\pm9.0}$&$74.1_{\pm9.4}$&$80.3_{\pm10.8}$&$82.0_{\pm6.1}$&$95.3_{\pm4.7}$&$91.6_{\pm8.8}$&$73.0_{\pm11.0}$&$81.8_{\pm6.2}$&$90.0_{\pm3.6}$&$79.2_{\pm14.6}$&$201.3_{\pm168.6}$ \\
            \bottomrule
        \end{tabular}
    }
    \end{minipage}
\end{sidewaystable}

%%%%%%%%%%%%%%%%%%%%%%%%%%%%%%%%%%%%%%%%%%%%%%%%%%%%%%%%%%%%

\end{document}